\documentclass{article}

\usepackage[preprint]{neurips_2025}

\usepackage[utf8]{inputenc}
\usepackage[T1]{fontenc}
\usepackage{url}
\usepackage{booktabs}
\usepackage{amsfonts}
\usepackage{amssymb}
\usepackage{amsmath}
\usepackage{amsthm}
\usepackage{mathtools}
\usepackage{nicefrac}
\usepackage{microtype}
\usepackage{comment}
\usepackage{enumitem}
\usepackage{times}
\usepackage{soul}
\usepackage[small]{caption}
\usepackage{graphicx}
\usepackage[dvipsnames]{xcolor}
\usepackage{latexsym}
\usepackage{bm}
\usepackage{bbm}
\usepackage{algorithm}
\usepackage[colorlinks=true,linkcolor=blue,citecolor=blue]{hyperref}
\usepackage{cleveref}
\usepackage{algorithmic}
\usepackage{xparse}
\usepackage{thm-restate}
\usepackage{tabularx}

\usepackage{tikz}
\usetikzlibrary{calc}

\theoremstyle{plain}
\newtheorem{theorem}{Theorem}[section]

\newtheorem{proposition}[theorem]{Proposition}

\theoremstyle{definition}
\newtheorem{definition}[theorem]{Definition}
\newtheorem{remark}[theorem]{Remark}
\newtheorem{example}[theorem]{Example}

\newcommand{\R}{\mathbb{R}}
\newcommand{\calA}{\mathcal{A}}

\newcommand{\calS}{\mathcal{S}}

\newcommand{\calN}{\mathcal{N}}

\newcommand{\calP}{\mathcal{P}}

\newcommand{\calE}{\mathcal{E}}
\newcommand{\calV}{\mathcal{V}}

\newcommand{\MCL}{\lambda}

\newcommand{\ind}{\mathbbm{1}}
\renewcommand{\Pr}{\mathbb{P}}
\newcommand{\E}{\mathbb{E}}
\newcommand{\Cov}{\operatorname{Cov}}
\newcommand{\Var}{\operatorname{Var}}
\newcommand{\Div}{\operatorname{Div}}
\newcommand{\Bor}{\operatorname{Bor}}
\newcommand{\pl}{\operatorname{pl}}

\newcommand{\splur}[2]{p_{#1}{(#2)}}
\newcommand{\splurVar}[1]{p_{#1}}
\newcommand{\rankings}{\Pi_m}

\newcommand{\piAN}{\pi_{\mathrm{AN}}}
\newcommand{\pimin}{\pi_{\mathrm{min}}}
\newcommand{\pisym}{\pi_{\mathrm{sym}}}
\newcommand{\piic}{\pi_{\mathrm{IC}}}
\usepackage{titletoc}

\newif\ifredactionMode
\newif\ifhideDone
\hideDonefalse
\redactionModetrue

\ifredactionMode
    
    \ifhideDone
        \newcommand{\resolved}[1]{}
    \else
        \newcommand{\resolved}[1]{#1}
    \fi
    \newcommand{\TODO}[1]{\textcolor{Red}{[\textsf{TODO}: #1]}}
    
    \newcommand{\FG}[1]{\textcolor{Plum}{[\textsf{Felipe}: #1]}}
    \newcommand{\MT}[1]{\textcolor{ForestGreen}{[\textsf{Magdalena}: #1]}}
    \newcommand{\MO}[1]{\textcolor{RoyalBlue}{[\textsf{Mohamed}: #1]}}
    \newcommand{\UG}[1]{\textcolor{BurntOrange}{[\textsf{Umberto}: #1]}}
    \newcommand{\RV}[1]{\textcolor{Purple}{[\textsf{Reviewer}: #1]}}
    \newcommand{\re}[1]{%
        \par\noindent
        \hspace*{1.5em}%
        \textcolor{gray}{\vrule width 1.5pt}%
        \hspace{0.6em}%
        \begin{minipage}[t]{\dimexpr\linewidth-1.5em-1.5pt-0.6em\relax}%
        #1%
        \end{minipage}%
    }
\else
    
    \newcommand{\resolved}[1]{}
    \newcommand{\TODO}[1]{}
    
    \newcommand{\FG}[1]{}
    \newcommand{\MT}[1]{}
    \newcommand{\MO}[1]{}
    \newcommand{\UG}[1]{}
    \newcommand{\RV}[1]{}
    \newcommand{\re}[1]{}
\fi

\def\stickplotwidth{5.5}

\def\stickheightratio{0.20}
\def\stickticksizeratio{0.012}

\definecolor{colIC}{HTML}{222222}
\definecolor{colComp}{HTML}{C8322B}
\definecolor{colAxis}{HTML}{444444}

\newcommand{\stripframe}[3]{%
  \pgfmathsetmacro{\yoff}{#1}%
  \pgfmathsetmacro{\spTick}{\stickplotwidth*\stickticksizeratio}%
  \node[anchor=east, font=\small] at (0.5, \yoff) {#3};
  \draw[line width=0.5pt, colAxis] (1, \yoff) -- (#2, \yoff);
  \foreach \j in {1,...,#2} {
    \draw[line width=0.5pt, colAxis]
      (\j, \yoff) -- (\j, \yoff + \spTick);
  }
  \foreach \j in {1,3,...,#2} {
    \node[below=1pt, font=\scriptsize, colAxis]
      at (\j, \yoff) {\j};
  }
}

\newcommand{\compstrip}[5]{%
  \stripframe{#1}{#2}{#4}%
  \pgfmathsetmacro{\yoff}{#1}%
  \pgfmathsetmacro{\spH}{\stickplotwidth*\stickheightratio}%
  \pgfmathsetmacro{\maxw}{#5}%
  \foreach \pos/\wt/\lbl in {#3} {
    \pgfmathsetmacro{\hgt}{\yoff + \spH*\wt/\maxw}%
    \draw[colComp, line width=1.5pt] (\pos, \yoff) -- (\pos, \hgt);
    \node[above=0.5pt, font=\tiny, colComp] at (\pos, \hgt) {\lbl};
  }
}

\newcommand{\compfig}[4]{%
  \begin{tikzpicture}[baseline=0]
    \pgfmathsetmacro{\xunit}{\stickplotwidth / (#1 - 1)}%
    \tikzset{x=\xunit cm, y=1cm}
    \compstrip{0}{#1}{#2}{#3}{#4}
  \end{tikzpicture}%
}

\title{Efficient Elicitation of Collective Disagreements}

\author{%
  Mohamed Ouaguenouni$^{1}$ \qquad Felipe Garrido-Lucero$^{1}$ \qquad  Umberto Grandi$^{1}$  \\
  \textbf{César Hidalgo}$^{2,3,4}$ \qquad \textbf{Magdalena Tydrichova}$^{5}$\\
  \\
  \centerline{\begin{minipage}{0.85\textwidth}
  {\small 1. IRIT, Université Toulouse Capitole, 
  Toulouse, France\\
  2. Center for Collective Learning, IAST, Toulouse School of Economics, Toulouse, France \\
  3. Center for Collective Learning, CIAS, Corvinus University of Budapest, Budapest, Hungary \\
  4. AMBS, University of Manchester, Manchester, UK\\
  5. Centrale Supélec, Paris Saclay, France}
  \end{minipage}}
}
\begin{document}

\maketitle

\begin{abstract}
We analyze the structure of the disagreement among a population of voters over a set of alternatives. Surveys typically ask either for pairwise comparisons, simple and intuitive for participants, or full rankings over alternatives, eliciting the entire voters' preferences. Building on the observation that pairwise comparisons cannot distinguish structural disagreement from noise, we propose a stratified framework to identify the minimal aggregated preference information needed to compute a number of disagreement measures from the literature. 
%In a digital democracy setting, where the number of alternatives is large and the attention span of voters is scarce, our framework provides a viable alternative to measuring disagreements through full rankings. 
Specifically, we introduce the plurality matrix, a generalization of pairwise comparisons that records, for every subset $S$ of alternatives, the probability that each $a \in S$ ranks first in $S$. We define the level of a disagreement measure as the smallest subset size needed to express it, showing that many existing notions, including rank-variance and divisiveness, sit at level $3$, proving that pairwise comparisons are not enough. In addition, we demonstrate the interest of going beyond level $3$ both theoretically and experimentally. To make these results actionable, we design two elicitation protocols to estimate the plurality matrix,
%, namely the $k$-chain and the $k$-ranking, 
exploring the trade-off between the number of required participants and the cognitive load requested to each of them.
%information queried from each participant.
%We analyze the structure of the disagreement among a population of voters over a set of alternatives. Surveys typically ask either for pairwise comparisons, simple and intuitive for participants, or full rankings over the alternatives, which encompass the whole structure of the voters' preferences. Observing that pairwise comparisons cannot distinguish structural disagreement from noise, we propose a stratified framework to identify the minimal aggregated preference information needed to compute a number of disagreement measures. In a digital democracy setting, where the number of alternatives is large and the attention span of voters is scarce, our framework provides a viable alternative to measuring disagreements through full rankings. We introduce the plurality matrix, a generalization of pairwise comparisons that records, for every subset $S$ of alternatives, the probability that each $a \in S$ ranks first in $S$. We define the level of a disagreement measure as the smallest subset size needed to express it, showing that many existing measures, including rank-variance and divisiveness-alike measures, sit at level $3$, proving that pairwise comparisons are not enough. In addition, we demonstrate the interest of going beyond level $3$ both theoretically and experimentally. To make these results actionable, we design two protocols to elicit the plurality matrix, namely the $k$-chain and the $k$-ranking, exploring the trade-off between the number of required participants and the information queried from each participant.
\end{abstract}

\section{Introduction}\label{sec:introduction}

Social choice theory seeks to model and improve how a society 
%composed of self-driven individuals 
makes collective decisions. A large part of the literature addressed this question by proposing preference aggregation rules, which aim to select the most preferred alternative(s) from a set of options given the agents’ declared preferences, by characterizing them axiomatically, and by designing efficient elicitation protocols for their computation \citep{arrow2010handbook,HandbookCOMSOC,furnkranz2017preference}. 
However, consensus is only half of the story: society must often confront and resolve disagreements before reaching agreement~\citep{waldron1999law}.
%to reach agreements a society has to pass through a number of iteration to identify and solve their disagreements first \citep{waldron1999law}.
This is well illustrated by the deliberative platform Pol.is~\citep{small2021polis,small2023opportunities}, where proposals with the narrowest voting margins are highlighted alongside consensual ones. 
Despite growing interest, relatively little attention has been devoted to disagreements in the (computational) social choice literature. Intuitively, an alternative is divisive when a group of individuals favors it while others oppose it. Still, a canonical definition has not yet been established: existing approaches rely on polarization measures \citep{esteban_ray}, the variance of an alternative’s ranking across users \citep{gaitonde2020adversarial1,musco2018minimizing}, or relative rankings to identify  population factions \citep{navarrete2024understanding}.% with existing approaches drawing on classical polarization measures \citep{esteban_ray}, considering the variance of an alternative’s ranking across users \citep{musco2018minimizing,gaitonde2020adversarial1} or using the relative ranking of alternatives to identify factions in the voting population \citep{navarrete2024understanding}.

Beyond the choice of aggregation rules, an equally important aspect of collective decision-making lies in how preferences are elicited. 
%The quality and structure of the collected preferences fundamentally shape the outcomes of any aggregation procedure. 
A substantial body of work has focused on designing efficient elicitation protocols, aiming to reduce cognitive and communication burdens while preserving enough information for accurate decision-making \citep{conitzer2005communication,procaccia2009thou}. A particularly simple and widely used elicitation primitive
%\footnote{Primitive in the sense that more complex methods can be constructed from it.} 
is that of \textit{pairwise comparisons} \citep{thurstone1927method}, where agents are asked to choose between two alternatives. However, relying only on pairwise comparisons may fail to reveal important features of collective preferences, in particular those related to disagreement and polarization, as illustrated in the following example: %Interestingly, a first key observation of our article is that pairwise comparisons may fail to identify the divisive alternatives, as we illustrate in the following example.

\textbf{Pairwise comparisons are not enough to capture collective disagreement.} Consider two populations ranking three alternatives $a,b,c$. In Population~1, all six possible rankings appear equally often. In Population~2, voters are split into two opposing groups: half rank $a$ first and $c$ last, while the other half do the opposite. From a pairwise comparison perspective, these two populations are identical: for any pair of alternatives, each is equally likely to be ranked above the other. However, the underlying structures of disagreement are very different. In Population~1, the position of every alternative is spread uniformly over all three ranks, while Population 2 is split with alternatives $a$ and $c$ being ranked equally either first or last.

%In this article, we introduce a formal framework for disagreement measures and their interaction with preference elicitation, focusing on how the choice of elicitation protocols influences the identification and characterization of divisive alternatives. \MT{Not sure we are actually doing which preceeds (the end of previous sentence) !} By jointly analyzing elicitation and aggregation, we aim to uncover their complementarity and better understand how information acquisition shapes the detection and interpretation of disagreement in collective decision-making.

\textbf{Contributions.} Building on the previous observation, %observation that elicitation schemes more expressive than those based solely on pairwise comparisons may be required to capture societal disagreement, 
we introduce a novel preference representation that encodes, for each subset of alternatives, the likelihood that each member of the subset is its top-ranked alternative. Leveraging this representation, we study the minimal information required to compute several well-known disagreement measures from the literature and design elicitation protocols suited to different settings. Our main contributions are the followings:
\vspace{-0.2cm}
\begin{itemize}[leftmargin = *]\setlength\itemsep{0.01cm}
\item We introduce the \textit{plurality matrix}, which, for each subset of alternatives $S$ and each $a \in S$, records the probability that $a$ is ranked first within $S$. The rows are ordered by increasing size of $S$, defining different \textit{degrees}, with pairwise comparisons corresponding to degree $2$.
\item We show that computing the \textit{agreement index} \cite{faliszewski2023}, the \textit{rank variance} \cite{kendall1939problem}, and the divisiveness measure \cite{navarrete2024understanding} requires information of degrees $2$, $3$, and $3$, respectively, demonstrating that some disagreement measures inherently require more information than pairwise comparisons.

\item We prove that, for any degree $k$, there exist instances with identical information up to degree $k$ but differing at degree $k+1$, showing that the hierarchy of degrees is strict. Moreover, we construct, for every degree $k$, a disagreement measure of exactly that degree.

\item We show that under single-peaked preferences or the Plackett--Luce model, the hierarchy collapses to degree $2$.

\item We propose two elicitation protocols and analyze the trade-off they induce between per-voter cognitive load and population size.
\end{itemize}
\textbf{Related work}. The related literature spans three complementary directions: measuring disagreement in social choice, enhancing elicitation beyond pairwise comparisons, and designing cost-efficient elicitation protocols. First, a number of works in social choice have studied how to define measures orthogonal to agreement. Existing approaches typically capture either the diversity of voters’ preferences \citep{Ammann2025PreferenceD, faliszewski_k-kemeny, hashemi2014measuring, karpov2017} or polarization \citep{can2015measuring,faliszewski2023}. These notions are generally defined at the level of full preference profiles and often rely on complete rankings. Closer to our perspective are measures defined at the level of alternatives, such as the divisiveness measure of Navarrete et al.~\cite{colley2023measuring, navarrete2024understanding}, and the axiomatic analysis of conflicting pairs by Delemazure et al.~\cite{delemazure2024conflicting}.
Second, several works consider richer query models than pairwise comparisons for preference elicitation and are closely related to our framework. In particular, subset-wise queries, where a set of alternatives is presented and the best one is returned, have been studied in learning and ranking settings \cite{pmlr-v89-saha19a}. Related models include top-$k$ queries, which extract more information per interaction \cite{ayadi2022approximating, chen2018_topk_ranking}, as well as recent work on eliciting full rankings over queried subsets of alternatives \cite{HalpernEtAlEC2024}. 
Third, a complementary line of work studies elicitation protocols under cognitive constraints and relates to our goal of designing viable alternatives to full ranking elicitation. In voting, the recent work of Terzopoulou \cite{terzopoulou2023voting} introduces the notion of voters’ energy in reporting preferences, capturing the limitation on the number of alternatives that can be ranked by each voter, and studies the resulting loss in social welfare for plurality and Borda rules. More broadly, behavioral decision theory highlights the impact of cognitive load and task complexity on response quality \cite{Griffin2005}, motivating the design of hybrid or adaptive elicitation formats that balance informational richness with practical usability \cite{Brams2009, dery2024interactive}.

\section{The Model}\label{sec:model_plur_matrix}
%%%%%%%%%%%%%%%%%%%%%%%%%%%%%%%%%%%%%%%%%%%%%%%%%%

This section is devoted to introducing a probabilistic model of agents' preferences over alternatives, to define the \textit{plurality matrix}, and to define three disagreement measures as our use-cases. 

Let $\calA = \{a, b, \ldots\}$ be a finite set of $m \geq 3$ alternatives, $\rankings$ the set of all $m!$ strict rankings of $\calA$, and $\calS := \{S \subseteq \calA : 2 \leq |S| \leq m\}$ (sorted by size, then lexicographically). For a statement $\tau$, $\mathbbm{1}\{\tau\}$ denotes the indicator of $\tau$. A \textbf{preference profile} (or voter population) is a probability distribution $\pi$ on $\rankings$, and a \textbf{voter} corresponds to a ranking ${\succ} \in \rankings$ sampled from $\pi$.

%%%%%%%%%%%%%%%%%%%%%%%%%%%%%%%%%%%%%%%%%%%%%%%%%%%%%%%%%%%
%\subsection{Preference profiles, rank, and Borda scores}

%Given $\calA = \{a,b,...\}$ a finite set of $m \geq 3$ alternatives, we denote $\rankings$ the set of all $m!$ strict rankings of $\calA$, and the set $\mathcal{S} := \{S \subseteq \calA : 2\leq |S|\leq m\}$, sorted by size, and within each size, lexicographically. Given a statement $\tau$, we denote $\mathbbm{1}\{\tau\}$ the indicator function that takes value $1$ if $\tau$ is true, and $0$ otherwise.
% \MT{Maybe remove " and $\succ$ a typical element of $\rankings$"? (We redefine it in the definition just below). }
% \MT{Also strictly speakings, sets are not ordered in maths. But it's highly understandable so why not.}

%\begin{definition}\label{def:preference_profile_and_voter}
%A \textbf{preference profile} (or voters population) corresponds to any probability distribution $\pi$ on $\rankings$. A \textbf{voter} is defined as an element $\succ\ \in \rankings$ sampled from $\pi$. 
%\end{definition}

We adopt a probabilistic view as many applications, e.g. online platforms, involve in general a large number of participants.
%We choose to model preference profiles as distributions as, in general, e-platforms count with a large number of participants. 
We now introduce one of the main concept of our paper, the \textit{plurality matrix}, which will serve as the basis for first analyzing disagreement and later designing elicitation protocols.
%to be the base for the analysis of disagreement and the posterior design of elicitation protocols.

\begin{definition}\label{def:plurality_matrix}
Let $\pi$ be a preference profile. We define the \textbf{plurality matrix} $\calP_\pi \in [0,1]^{\mathcal{S}\times \calA}$, where, for each $S \in \mathcal{S}$ and $a \in \calA$, the entry $(S,a)$ is given by,
\begin{align*}
    p^\pi_S(a) := \mathbb{P}_\pi(a\text{ is ranked first among the alternatives in }S) \cdot \mathbbm{1}\{a\in S\}.
    % p^\pi_S(a) := \left\{ \begin{array}{lc}
    %    \mathbb{P}_\pi(a\text{ is ranked first among the alternatives in }S) & \text{if }a \in S, \\
    %     0 & \text{if } a \notin S.
    % \end{array}\right.
\end{align*}
For simplicity, we write $\splurVar{xywz}^\pi$ instead of $\splurVar{\{x,y,w,z\}}^\pi$. Given $k \in \{2,...,m\}$, we refer to the \textbf{data or information of degree} $k$ to the sub-matrix containing only the rows related to sets of size $k$. We say a real-value statistic $\Phi$ is of \textbf{level} $k$ if for any preference profile $\pi$ and alternative $a$, $\Phi(\pi,a)$ can be expressed using data of degree at most $k$ but not at most $k-1$.
% , $S \subseteq \calA$, with $|S| \geq 2$, and $a \in S$ be fixed. We define the $S$\textbf{-plurality quantity of} $a$ as,
% \begin{align*}
%   \splurVar{S}^\pi(a) := \mathbb{P}_\pi(a\text{ is ranked first among the alternatives in }S) = \sum\nolimits_{\substack{\succ \in \rankings,r_a(\succ ) = 1}} \pi_S(\succ),
% \end{align*}
% where $\pi_S(\succ)$ corresponds to $\pi$ restricted to $S$. For any $k \in \{2,...,m\}$, we call the \textbf{degree-$k$ quantity} to the collection of all $S$-plurality quantities with $S\subseteq\calA$ such that $|S| = k$.
\end{definition}

We assume throughout the article that $\splurVar{S}^\pi(a) > 0$ for every profile $\pi$, every $S \in \calS$, and $a \in S$: an alternative violating this is 
%systematically 
dominated on every voter's ranking, hence never a locus of disagreement. For simplicity, we make one exception in Example~\ref{ex:IC_AN_require_degree_3_to_see_difs} below. We write $\splurVar{S}^\pi(a)= \cdot$ whenever $a\not\in S$.

\begin{definition}\label{def:rank_of_a_and_borda}
Let $\pi$ be a preference profile, $\succ \, \in \, \rankings$ be a ranking, and $a \in \calA$. We define the \textbf{rank} of $a$ and the \textbf{Borda score} of $a$, respectively by,
\begin{align*}
r_a(\succ) := 1 + |\{b \in \calA \setminus \{a\} : b \succ a\}|
\text{ and }
\Bor_\pi(a) := \sum\nolimits_{b \in \calA \setminus \{a\}} p_{ab}^\pi(a) = m - \E_\pi[r_a].
\end{align*}
Moreover, given $x,y\!\in\!\calA$, we define the Borda score of $a$ over the subpopulation preferring $x$ to $y$:
\begin{align*}
\Bor_\pi\bigl(a;\calN^{x \succ y}_\pi\bigr) := \sum\nolimits_{b \in \calA \setminus \{a\}} \mathbb{P}_\pi(a \succ b \mid x \succ y),
\end{align*}
%to be 
% Its expectation is directly related to the borda score: $\Bor(a) := \sum_{b \neq a} p_{ab}$ via $\E_\pi[r_a] = m - \Bor(a)$. 
\end{definition}

Pairwise comparisons (more precisely, their aggregation over all voters) correspond to the data of degree $2$. In particular, the Borda score is a measure of level $2$. Pairwise comparisons are widely used in the literature to compute many voting rules \citep{tournamentHandbook2016,weight-tournamentHandbook2016,navarrete2024understanding}. At the other extreme, data of degree $m$ correspond to plurality scores, i.e. the proportion of voters ranking each alternative first. In general, data of different degrees in the plurality matrix are unrelated, in the sense that none can be recovered from the other. %Moreover, as we will see in \Cref{sec:primitives}, the data of different degrees differ both in the aspects of the preference profiles they capture, as well as in their elicitation cost. This observation will play a key role in our future analysis. We give a first insight of this property in the following example. 
We illustrate the notion of plurality matrix and give some first insights next.
% in the following example:

\begin{example}\label{ex:IC_AN_require_degree_3_to_see_difs}
Let $\calA = \{a, b, c\}$ and consider two preference profiles:\footnote{Note the two profiles sit at opposite poles of the visualization compass of Szufa et al. \cite{szufa2025drawing}.} the impartial-culture profile $\piic$,\footnote{Preference profiles in our terminology are typically called \emph{cultures} or \emph{distributions} in the social choice literature. In this example we consider the impartial culture, but see Figure~\ref{fig:synthetic} for profiles corresponding to other distributions.}
%instances of the Mallow, Plackett-Luce, $k$-Euclidean and single-peaked distributions.} 
where all alternatives are uniformly ranked, and the antagonism profile $\piAN$, where $b$ and $c$ are uniformly ranked but $a$ is ranked either first or last, each with probability $1/2$. Formally, $r_a$ verifies,
\begin{align*}
 \Pr_{\piic}[r_a = i] = \tfrac13, \text{ for } i \in \{1, 2, 3\} \text{ and } \Pr_{\piAN}[r_a = i] = \tfrac12\cdot\ind\{i=1\} + \tfrac12\cdot\ind\{i=3\}   
\end{align*}
% $\Pr_{\piic}[r_a = i] = 1/3$ for $i \in \{1, 2, 3\}$ and $\Pr_{\piAN}[r_a = i] = \tfrac12\ind\{i=1\} + \tfrac12\ind\{i=3\}$.
\Cref{tab:table_example_plurality_matrix} shows the plurality matrix for the two aforementioned profiles.
\vspace{-0.2cm}

\begin{table}[h]
\caption{Plurality matrices for $\piic$ and $\piAN$ over $\calA = \{a,b,c\}$}
\label{tab:table_example_plurality_matrix}
\centering
\begin{tabular}{c|ccc}
\toprule
\multicolumn{4}{c}{$\calP_{\piic}$}\\
\midrule
$\calS \times \calA$ & $a$ & $b$ & $c$  \\
\midrule
$\{a,b\}$   & $\nicefrac12$ & $\nicefrac12$ & $\cdot$\\
$\{a,c\}$   & $\nicefrac12$ & $\cdot$           & $\nicefrac12$ \\
$\{b,c\}$   & $\cdot$ & $\nicefrac12$ & $\nicefrac12$ \\
$\{a,b,c\}$ & $\nicefrac13$ & $\nicefrac13$ & $\nicefrac13$ \\
\bottomrule
\end{tabular}\qquad
\begin{tabular}{c|ccc}
\toprule
\multicolumn{4}{c}{$\calP_{\piAN}$}\\
\midrule
$\calS \times \calA$ & $a$ & $b$ & $c$  \\
\midrule
$\{a,b\}$   & $\nicefrac12$ & $\nicefrac12$ & $\cdot$\\
$\{a,c\}$   & $\nicefrac12$ & $\cdot$           & $\nicefrac12$ \\
$\{b,c\}$   & $\cdot$ & $\nicefrac12$ & $\nicefrac12$ \\
$\{a,b,c\}$ & $\nicefrac12$ & $\nicefrac12$ & $0$ \\
\bottomrule
\end{tabular}
\smallskip
\end{table}
\vspace{-0.2cm}
Notice that the two matrices coincide on all entries of degree $2$ but differ at degree $3$. Any measure based solely on pairwise comparisons, therefore, will not be able to distinguish them. Yet, both profiles clearly differ in terms of disagreement.
\end{example}

The plurality matrix will be shown to be rich enough to express disagreement notions studied in this paper (as well as several pairwise-based social choice concepts), yet simple enough to define efficient elicitation protocols under limited cognitive load. 
Moreover, by working with aggregate proportions, we adopt a privacy-preserving approach that does not require storing individual preferences.

In the rest of the article, unless required, we will drop the dependence of all defined terms on the preference profile $\pi$ for simplicity. 

\subsection{Three Measures of Disagreement}
\label{subsec:div}

We introduce three measures of disagreement from the literature to be the use-cases of the article, namely, the \textit{agreement index} of 
%Can et al. \cite{can2015measuring} (as presented by 
Faliszewski et al. \cite{faliszewski2023}, the \textit{rank variance} of Kendall and Smith \cite{kendall1939problem}, and the \textit{divisiveness} of Navarrete et al. \cite{navarrete2024understanding}.

% \UG{Please do not re-name N-divisiveness, there is no need and it will just create confusion in the literature. Just call it divisiveness.}

\begin{definition}\label{def:agreement_index}
Given a profile $\pi$, we define its \textbf{agreement index} as
\begin{align*}
  A(\pi) := \frac{1}{\binom{m}{2}}\sum\nolimits_{\{x,y\} \subseteq \calA} \bigl|2\cdot \splur{xy}{x} - 1\bigr|.
\end{align*}
\end{definition}
\vspace{-0.2cm}
The agreement index aggregates pairwise comparisons into a single scalar ranging from $0$ (all pairwise comparisons are perfectly split) to $1$ (unanimity on all pairs). Notice that Definition \ref{def:agreement_index} is entirely determined by pairwise comparisons, thus, as illustrated in \Cref{ex:IC_AN_require_degree_3_to_see_difs}, it might fail to detect any disagreement that is not already visible at the pairwise level. 
% The following measures work at a local level.

% \begin{definition}\label{def:rank_variance}
% Let $\pi$ be a preference profile and $a\in\calA$ an alternative. We define the \textbf{rank variance} of $a$ as $\Var_\pi(a) := \E_\pi[(r_a\!-\!\E_\pi[r_a])^2]$.
% %\begin{align*}
% %  \Var_\pi(r_a) \;:=\; \E_\pi\bigl[(r_a - \E_\pi[r_a])^2\bigr].
% %\end{align*}
% \end{definition}
% The rank variance quantifies how much the rank of $a$ fluctuates across the population. 

% \begin{definition}\label{def:n-divisiveness}
% Let $\pi$ be a preference profile and $a \in \calA$ an alternative. We define the \textbf{divisiveness} of $a$ as
% \begin{align*}
% &\Div_\pi(a) := \frac{1}{m-1}\sum\nolimits_{b \in\calA\setminus \{a\}}
%     \bigl|\Bor_\pi\bigl(a;\calN^{a \succ b}_\pi\bigr)
%          \;-\; \Bor_\pi\bigl(a;\calN^{b \succ a}_\pi\bigr)\bigr|.
%          % \text{ where},\\
% % &\Bor\bigl(a;\calN^{x \succ y}_\pi\bigr) := \sum\nolimits_{b \in \calA \setminus \{a\}} \mathbb{P}_\pi(a \succ b \mid x \succ y), \text{ for any } x,y \in \calA.
% \end{align*}
% \end{definition}

\begin{definition}\label{def:rank_variance}
Let $\pi$ be a preference profile and $a\in\calA$ an alternative. 
\vspace{-0.2cm}
\begin{itemize}[leftmargin = *]\setlength\itemsep{-0.1cm}
\item We define the \textbf{rank variance} of $a$ as $\Var_\pi(a) := \E_\pi[(r_a\!-\!\E_\pi[r_a])^2]$.
\item We define the \textbf{divisiveness} of $a$ as
\begin{align*}
&\Div_\pi(a) := \frac{1}{m-1}\sum\nolimits_{b \in\calA\setminus \{a\}}
    \bigl|\Bor_\pi\bigl(a;\calN^{a \succ b}_\pi\bigr)
         \;-\; \Bor_\pi\bigl(a;\calN^{b \succ a}_\pi\bigr)\bigr|.
         % \text{ where},\\
% &\Bor\bigl(a;\calN^{x \succ y}_\pi\bigr) := \sum\nolimits_{b \in \calA \setminus \{a\}} \mathbb{P}_\pi(a \succ b \mid x \succ y), \text{ for any } x,y \in \calA.
\end{align*}
\end{itemize}
\end{definition}
The rank variance quantifies how much the rank of $a$ fluctuates across the population. Following Navarrete et al. \cite{navarrete2024understanding}, we consider divisiveness combined with Borda scores (although, note that any scoring function could be used, such as the Copeland score). Intuitively, for each alternative $b \in \calA\setminus\{a\}$, divisiveness partitions the population into those preferring $a$ to $b$ and those preferring $b$ to $a$, and compares the (Borda) score of $a$ across these two groups. We conclude this section by illustrating the three introduced measures of disagreement.

\begin{example}\label{ex:example_measures}
Consider an instance with $|\calA| = 15$ alternatives and $a \in \calA$ fixed. Consider the two preference profiles $\piic$ and $\piAN$ as defined in Example \ref{ex:IC_AN_require_degree_3_to_see_difs}, extended to $15$ alternatives. In addition, consider the preferences profiles $\pimin$ and $\pisym$ such that, for any $b \in \calA\setminus\{a\}$, the rank of $b$ is uniformly distributed, and, given $\varepsilon_1, \varepsilon_2 > 0$,
\begin{align*}
    &\mathbb{P}_{\pimin}[r_a = i] = \varepsilon_1\cdot\ind\{i = 1\} + (1-\varepsilon_1)\cdot\ind\{i = 15\},\\
    &\mathbb{P}_{\pisym}[r_a = i] = \varepsilon_2\cdot\ind\{i = 1\} + \varepsilon_2\cdot\ind\{i = 15\} + (1-2\varepsilon_2)\cdot\ind\{i = 8\}.
\end{align*}
In words, in $\piic$ all rankings are equally likely. In $\piAN$, half of the population ranks $a$ first while the other half ranks it last. In $\pimin$, most of the population ranks $a$ last while a small portion ranks it first. Finally, in $\pisym$, small populations rank $a$ respectively first and last, while most of the voters rank it in the middle. Figure \ref{fig:three-profiles} illustrates the four distributions.
\vspace{-0.2cm}
\begin{figure}[H]
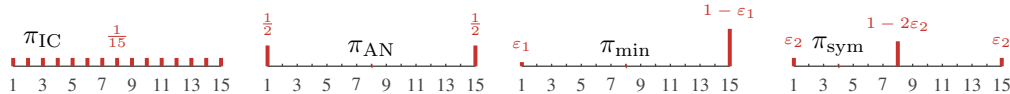

\centering
{\renewcommand{\stickplotwidth}{2.75}%
 \setlength{\tabcolsep}{2pt}%
 \compfig{15}%
  {1/0.2/{},
  2/0.2/{},
  3/0.2/{\normalsize{\textcolor{black}{$\piic$}}},
  4/0.2/{},
  5/0.2/{},
  6/0.2/{},
  7/0.2/{},
  8/0.2/{$\tfrac{1}{15}$},
  9/0.2/{},
  10/0.2/{},
  11/0.2/{},
  12/0.2/{},
  13/0.2/{},
  14/0.2/{},
  15/0.2/{}}%
  {}{1}\hspace{1pt}%
\compfig{15}%
  {1/0.5/{$\tfrac{1}{2}$},
  8/0.0/{\normalsize{\textcolor{black}{$\piAN$}}},
   15/0.5/{$\tfrac{1}{2}$}}%
  {}{1}\hspace{1pt}%
\compfig{15}%
  {1/0.1/{$\varepsilon_1$},
  8/0.0/{\normalsize{\textcolor{black}{$\pimin$}}},
   15/0.9/{$1-\varepsilon_1$}}%
  {}{1}\hspace{1pt}%
\compfig{15}%
  {1/0.2/{$\varepsilon_2$},
  4/0.0/{\normalsize{\textcolor{black}{$\pisym$}}},
   8/0.6/{$1-2\varepsilon_2$},
   15/0.2/{$\varepsilon_2$}}%
  {}{1}%
}
\caption{%
Distribution of $r_a$ for each profile represented as probability mass functions.}
\label{fig:three-profiles}
\end{figure}

\vspace{-0.5cm}
Table \ref{tab:measures-on-profiles} shows the metrics on the four profiles. As expected, the agreement index struggles to differ a uniform profile from profiles with split populations, while the other metrics manages to do it. However, rank variance is maximized at $\piAN$, as $a$ is equally often ranked first or last, while divisiveness is maximized for both $\piAN$ and $\pimin$, showing that only the presence, not the frequency, of extreme ranks might matter. %the number of voters ranking $a$ first and last might not play a role, but only the fact that $a$ is actually ranked first and last. 
Interestingly, for $\piic$ and $\pisym$, both metrics coincide. %Understanding how two different aggregations can arrive at identical values on such different profiles requires a framework that clarifies what each metric actually captures; we develop it in the next section.

\begin{table}[ht]
\centering
\caption{Agreement index, rank variance, and divisiveness of four profiles, with $\varepsilon_1 = 0.05$ and $\varepsilon_2 = \frac{4}{21}$.}
%with $\varepsilon_1 = \frac{1}{20}$ and $\varepsilon_2 = \frac{4}{21}$.}
\label{tab:measures-on-profiles}
\vspace{0.5em}
\renewcommand{\arraystretch}{1.1}
\setlength{\tabcolsep}{6pt}
\begin{tabular}{@{}l cccc@{}}
\toprule
 & $\piic$ & $\piAN$ & $\pimin$ & $\pisym$ \\
\midrule
$A(\pi)$             & $0$     & $0$   & $0.12$    & $0$    \\
$\Var_\pi(a)$ & $18.\bar{6}$  & $49$  & $9.31$ & $18.\bar{6}$ \\
$\Div_\pi(a)$       & $5.\bar{3}$  & $14$  & $14$      & $5.\bar{3}$ \\
\bottomrule
\end{tabular}
\end{table}

% \TODO{Adding to the caption: "Height represent the probability masses" or something like this}
%\Cref{tab:measures-on-profiles} reports $A$, $\Var_\pi(r_a)$, and
%$\Div(a)$ on the four profiles.
%The Agreement Index depends only on pairwise proportions and
%assigns $A = 0$ to both $\piic$ and~$\piAN$, which share
%$p_{ab} = 1/2$ for every opponent~$b$.
%The rank variance and Navarrete divisiveness separate these two
%profiles but order $\piic$ and $\pimin$ oppositely: the variance
%measures mass-weighted spread and prefers the uniform distribution
%over two concentrated atoms, whereas the divisiveness measures the
%gap between the two camps induced by a pairwise split and is
%size-blind, reaching its maximum~$14$ on both $\piAN$ and~$\pimin$.
%
%More striking, the profile $\pisym$ collides with $\piic$ on all three measures simultaneously.
%The rank distribution of~$\pisym$ is visibly split between a polarised minority at ranks $\{1, 15\}$ and an indifferent majority at rank~$8$, yet the variance and the divisiveness both compress this structure into a single scalar in which the two components cancel against each other: the signal of the extremes is diluted by the mass of the middle, and the two profiles become indistinguishable.
%Understanding how two different aggregations can arrive at identical values on such different profiles requires a framework that exposes what each measure actually reads from the preference data. We develop this framework in the next section.
\end{example}

\section{The Plurality Hierarchy}\label{sec:hiearchy}

This section is dedicated to study the hierarchy of the different degrees information of the plurality matrix. We begin by determining the level of the disagreement measures introduced in Section~\ref{subsec:div}, complemented by an extensive analysis in Appendix~\ref{app:master-inventory} of a large number of agreement, disagreement, and polarization measures from the literature. We then prove that the induced level hierarchy is strict and that it collapses under standard structural assumptions on voters' preferences.
%, presenting in Appendix~\ref{app:master-inventory} their formulas and levels in the plurality hierarchy.} 

\subsection{Level of Disagreement Measures}

Our first result is a straighforward corollary of the definition of the agreement index.

\begin{proposition}
The Agreement Index $A(\pi)$ is a measure of level $2$.
\end{proposition}
\begin{comment}
\begin{proof}
    Trivial: $A(\pi) = \tbinom{m}{2}^{-1}\sum_{\{x,y\}} |2\,p_{xy} - 1|$ and each $p_{xy} = \splur{\{x,y\}}{x}$ is of degree~$2$. 
\end{proof}
\end{comment}

%The Agreement Index $A(\pi) = \binom{m}{2}^{-1} \sum_{\{x,y\}} \bigl|2\,p_{xy} - 1\bigr|$ is a level-$2$ measure, since each $p_{xy} = \splur{\{x,y\}}{x}$ is a degree-$2$ entry of the plurality matrix. 

%This section is devoted to further study the different degrees of the plurality matrix. We begin by establishing the level of the three notions introduced in \Cref{subsec:div}.
% how the three use-case disagreement notions can be expressed through degree-$k$ quantities, and to establish a hierarchy of the different plurality quantities. 
%By definition, we have:

%\begin{proposition}
%The Agreement Index $A(\pi)$ is a measure of level $2$.
%\end{proposition}

Regarding the rank variance and the divisiveness, we prove two important results: first, they are fully captured by our framework and second, they require entries of degree strictly higher than two.

\begin{restatable}[]{proposition}{varsplur}\label{prop:var-splur}
The rank variance is a measure of level $3$. Moreover, for any preference profile $\pi$ and alternative $a \in \calA$, it holds,
\begin{align*}
  \Var(a)
  \;=\;
  \sum_{b \in \calA\setminus\{a\}}\!\! p_{ab}(a)\cdot(1 - p_{ab}(a))
  \;+\!\!
  \sum_{\substack{b,c \in \calA\setminus\{a\}, b \neq c}}
    \!\!\!\!\bigl(\,\splur{abc}{a} - p_{ab}(a)\cdot p_{ac}(a)\bigr).
\end{align*}
\end{restatable}

%\MT{Proof made more explicit, ready to go to appendix. \Done}

%\MT{Why "at most"? For me, it's exactly level 3, no? }
%\re{\MO{If we want to be perfectly rigourous, we defined the level as the smallest maximal degree that allows us to write the measure, and so to show that some measure is of level X we should prove that we can write it using quantities of degree at most X AND also show that if we used quantities to degree $X-1$ we can find two profiles such that the measures are different but the quantities up to level $X-1$ are the same, in the case of Divisiveness and Rank Variance IC and AN could play this role.}}
%\re{\MT{OK, I've understood - we only found formula that is level 3, which is upper bound (as stated below the last prove), that's why "at most" - we need to add "at most" in the statement of rank variance then ! And we then prove the tightness of these bounds.} }

\begin{restatable}[]{proposition}{divsplur}\label{prop:div-splur}
The divisiveness is a measure of level $3$. Moreover, for any preference profile $\pi$ and alternative $a\in \calA$, it holds,
\begin{align*}
    \Div(a) = \frac{1}{m-1}\sum_{b\in\calA\setminus\{a\}} \biggl| 1 + \sum_{c\in\calA\setminus\{a,b\}} \left[ p_{abc}(a)\cdot\biggl(\frac{1}{1-p_{ab}(a)} + \frac{1}{p_{ab}(a)}\biggr) - \frac{p_{ac}(a)}{1-p_{ab}(a)}\biggr] \right|.
\end{align*}
% \begin{align}
%   \E_\pi[r_a \mid a \succ b]
%   &\;=\; m - 1
%       \,-\, \frac{1}{p_{ab}}\cdot\sum_{c \in \calA\setminus\{a, b\}}\!\!\splur{abc}{a},
%       \label{eq:E-apb}\\
%   \E_\pi[r_a \mid b \succ a]
%   &\;=\; m
%       \,+\, \frac{1}{1 - p_{ab}}\cdot\sum_{c \in \calA\setminus\{a, b\}}
%             \bigl(\,\splur{abc}{a} - p_{ac}(a)\,\bigr).
%       \label{eq:E-bpa}
% \end{align}
% In particular, $\Div(a)$ is a measure of level at most $3$.   
\end{restatable}

The proofs of \Cref{prop:var-splur,prop:div-splur} are deferred to Appendix~\ref{app:mcl-lb-proof}. Note that, since rank variance and divisiveness can be expressed with entries of degree $2$ and $3$, both are of level at most $3$. To show tightness, consider $\piic$ and $\piAN$ as in \Cref{ex:example_measures}: they coincide at degree~$2$ ($p_{xy}(x) = \nicefrac{1}{2}$ for any $x,y\in\calA$), yet $\Var_{\piic}(a) \neq \Var_{\piAN}(a)$ and $\Div_{\piic}(a) \neq \Div_{\piAN}(a)$, as shown in \Cref{tab:measures-on-profiles}. Hence, neither measure can be expressed using only information of degree $2$.

% \subsection{Strictness of the hierarchy}

The previous observation actually extends beyond level $3$: for any integer $k \geq 2$, there exist two profiles that coincide on all degrees from $2$ to $k$, but differ on at least one entry of degree $k+1$. In particular, although the two profiles are different, no measure of level $k$ can distinguish them. 

% The following proposition provides an explicit construction of such pair of profiles for any $d \geq 2$. 

\begin{restatable}{proposition}{strictness}
\label{prop:hierarchy-strict}
For every $k \geq 2$, there exists $m \in \mathbb{N}$, a set $\calA$ of $m$ alternatives, and two preference profiles $\pi, \pi'$ such that $p_{S}^\pi(a) = p_{S}^{\pi'}(a)$ for any $S \subseteq \calA$ with $|S| \leq k$ and $a \in \calA$, while $p_{T}^\pi(a) \neq p_{T}^{\pi'}(a)$ for some $T \subseteq \calA$ with $|T| = k + 1$ and some $a \in \calA$. In particular, no measure of level $k$ distinguishes $\pi$ from $\pi'$.
\end{restatable}

The proof of \Cref{prop:hierarchy-strict} (in~Appendix \ref{app:mcl-lb-proof}) constructs the two stated profiles $\pi$ and $\pi'$. However, the question of whether a measure of level $k$, for any $k \in \mathbb{N}$, actually exists, rises. The answer is \textit{yes}.

\begin{definition}\label{def:central_moment}
Let $\calA$ be a set of $m$ alternatives. For any preference profile $\pi$ and $k \in \{1,...,m\}$, define the $k$\textbf{-central moment} $M_k^\pi$ as,
\begin{align*}
    M_k^\pi(a) := \E_\pi\bigl[(r_a - \E[r_a])^k\bigr], \text{ for any } a \in \calA.
\end{align*}
\end{definition}

The following result generalizes \Cref{prop:var-splur} to any degree. The proof is included in Appendix \ref{app:mcl-lb-proof}.

\begin{restatable}{theorem}{momentlevel}
\label{thm:moment-level}
The $k$-central moment $M_k$ is a measure of level $k+1$. Moreover, for any profile $\pi$ and alternative $a\in\calA$, it holds,
\begin{align*}\label{eq:moment-level}
  M_k(a) = (-1)^k \sum_{s=0}^{k}\ \ c_s(a) \cdot \!\!\!\!\! \sum_{{S \subseteq \calA \setminus \{a\}, |S| = s}}\!\!\!\splur{S \cup \{a\}}{a} = (-1)^k \sum_{s=0}^{k}\ \ c_s(a) \cdot\E_\pi\biggl[\binom{m-r_a}{s}\biggr],
\end{align*}
where $c_s(a) := \sum_{j=0}^{s} (-1)^{s-j} \binom{s}{j}(j - \Bor(a))^k$ and $p_a(a) = 1$.
\end{restatable}

The $k$-central moments, for different values of $k$, measure different properties of the distributions of the alternatives' ranks. Notably, given a preference profile $\pi$, for $k = 3$ and $k = 4$, we obtain, respectively, the \textbf{skewness} $\gamma_1^\pi$ and the \textbf{excess kurtosis} $\gamma_2^\pi$, formally given by,
\begin{align*}
    \gamma_1^\pi(a) := \frac{M_3^\pi(a)}{(M_2^\pi(a))^{\nicefrac{3}{2}}} \text{ and } \gamma_2^\pi(a) := \frac{M_4^\pi(a)}{(M_2^\pi(a))^{2}} - 3.
\end{align*}
The skewness measures the asymmetry of the distribution of $r_a$ while the excess kurtosis relates to tail extremity, reflecting the tendency of the distribution of $r_a$ to produce outliers. From \Cref{prop:hierarchy-strict} and \Cref{thm:moment-level}, $\gamma_1^\pi$ is a measure of level $4$ while $\gamma_2^\pi$ is a measure of level $5$. In particular, they give a two-dimensional summary of the rank distributions that no measure of level $2$ and $3$ can capture.

To better illustrate these measures, we present a numerical experiment on synthetic preference profiles from \cite{boehmer2024guide}. For an experiment on real data, see Appendix~\ref{sec:real_data_experiment}. In \Cref{fig:synthetic}, we consider an instance with $256$ alternatives and $7$ different preference profiles, namely, 
Mallows ($\varphi = 0.85$), Mallows mix-$2$ and mix-$4$ ($\varphi = 0.3$), Plackett-Luce with linear strengths, Walsh single-peaked, and $k$-Euclidean for $k \in \{2, 10\}$. \Cref{fig:synthetic} plots $(\gamma_1^\pi(x),\gamma_2^\pi(x))$ for each alternative $x\in\calA$ under each profile $\pi$ (represented by the colored dots). The solid parabola is the boundary of the \textit{Pearson inequality}, which states that for any $x \in \calA$, $\gamma_2^\pi(x)\geq( \gamma_1^\pi(x))^2-2$. The dashed parabola marks the transition between unimodal (above) and bimodal (below) rank distributions, the latter indicating higher disagreement. The figure has been zoomed on this transition region - in particular, some dots were left out of the image. Interestingly, ranks obtained from single-peaked, Plackett-Luce, and $k$-Euclidean distributions lie in the unimodal region, as they induce less disagreement among voters (see Appendix \ref{sec:synthetic_data} for more detailed discussion). In contrasts, Mallows mixtures shift toward the bimodal region, as alternatives ranked differently in the mixtures increase disagreement among voters.

Additionally, we plot $(\gamma_1^\pi(a),\gamma_2^\pi(a))$ for a fixed alternative $a \in \calA$ over five preference profiles $\piic,\pi_A,\pi_B,\pi_C$ and $\pi_D$. For each of them, all alternatives $b \in \calA\setminus\{a\}$ are uniformly ranked, as well as $a$ in $\piic$. The distribution of $r_a$ under the rest of the profiles is represented on the right-hand side of \Cref{fig:synthetic}. Up to degree $3$, it follows that all five distributions present the same plurality matrix. However, at degree $4$, the skewness is able to distinguish $\piic,\pi_C$, and $\pi_D$. Similarly, at degree $5$, the excess kurtosis distinguishes $\piic$, $\pi_A$, and $\pi_B$. In particular, for a practitioner using an elicitation model considering only information of degrees $2$ and $3$ will never be able to distinguish the unimodal distribution $\pi_A$ from the bimodal distributions $\pi_B,\pi_C,\pi_D$, from impartial culture $\piic$.

\begin{figure}[ht]
  \centering
  \includegraphics[width=\linewidth]{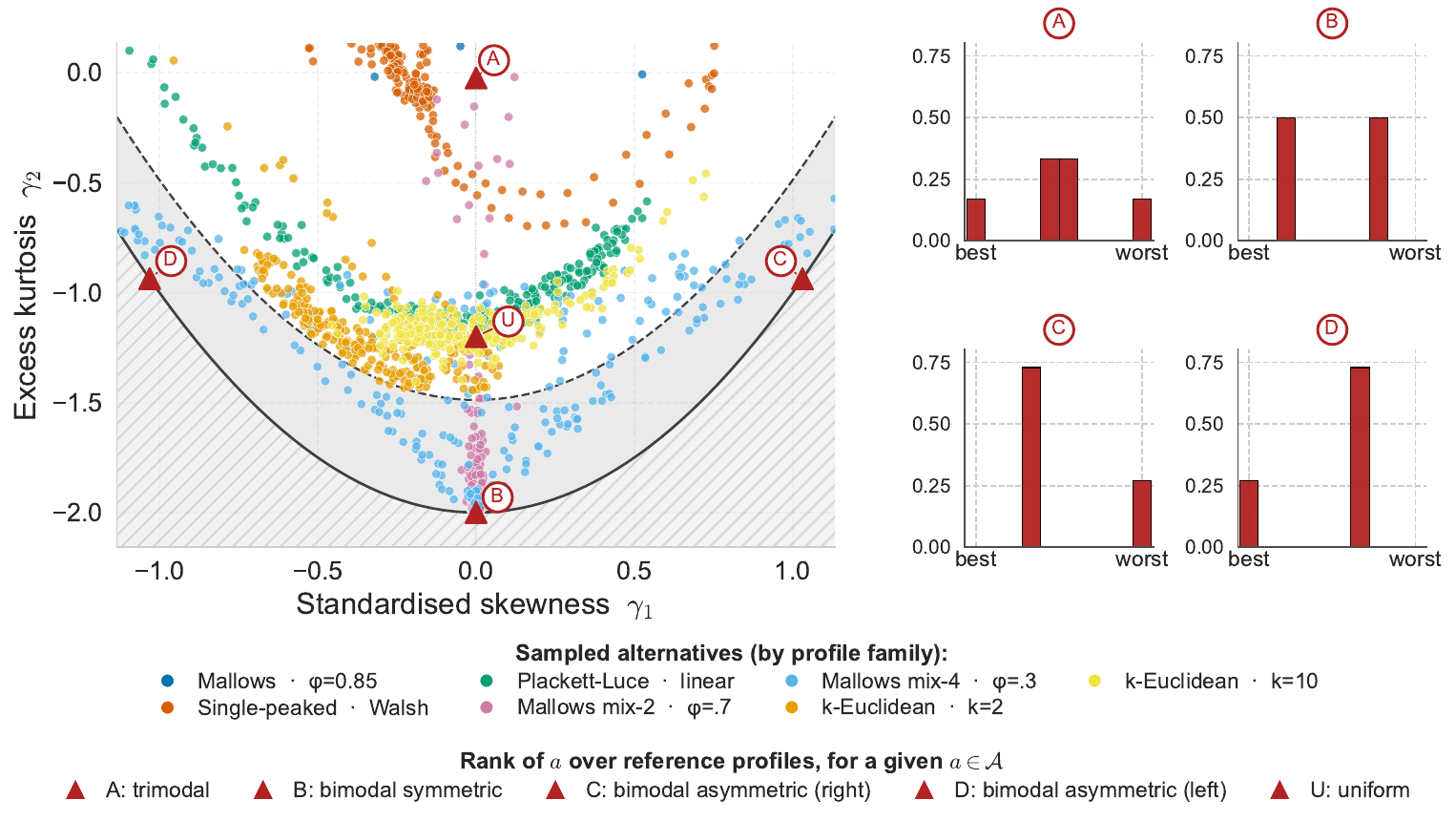}
    \caption{Skewness versus Excess kurtosis produced by seven different preferences profiles from \cite{boehmer2024guide} over $256$ alternatives. For each alternative $x \in \calA$, a colored dot represents $(\gamma_1^\pi(x),\gamma_2^\pi(x))$ for the respective profile $\pi$. The solid parabola represents the boundary of the Pearson inequality ($\gamma_2^\pi(x)\geq (\gamma_1^\pi(x))^2-2$ for any $x \in\calA$) while the dashed parabola marks the transition between unimodal (above) and bimodal (below) distributions. The triangles show $(\gamma_1^\pi(a),\gamma_2^\pi(a))$ for a fixed $a \in \calA$ whose rank is either uniform (triangle $U$) or displayed on the right-hand side. The plurality matrix of the five distributions $\piic,\pi_A,\pi_B,\pi_C$ and $\pi_D$ are identical up to degree $3$, so no measure of level $3$ can detect the difference between them. An interactive version of this plot is provided as supplementary material and presented in Appendix \ref{app:supp-viz}.}
  %\caption{Skewness versus Excess kurtosis of seven different preference profiles from \cite{boehmer2024guide} over $256$ alternatives. The solid parabola represents the boundary of the Pearson inequality ($\gamma_2(x)\geq \gamma_1(x)^2-2$ for any $x \in\calA$) while the dashed parabola marks the transition between unimodal (above) and bimodal (below) distributions. The triangles show $(\gamma_1(a),\gamma_2(a))$ for a fixed $a \in \calA$ whose rank is either uniform (triangle $U$) or displayed on the right-hand side. The plurality matrix of the five distributions are identical up to degree $3$, so no measure of level $3$ can detect the difference between them.}
  \label{fig:synthetic}
\end{figure}

\subsection{Disagreement Under Structural Assumptions}
%\subsubsection{Plackett--Luce}
% \MT{Moved some related references here, and commented the paragraph "let's decide later" in intro: I think it fits better here, as structural assumptions are not the core of the paper so the parahraph felt a bit odd in intro. Here it helps to motivate why we consider structural assumptions - it's done in other problems, and it often helps a lot, so let us see whqat happens in our setting (I tried to build this narrative, not 100 \% happy with what I've done, but as a first draft, the idea is there!)}

The strictness of the hierarchy shows that higher-degree information cannot, in general, be recovered from lower-degree observations. However, this perspective may be overly pessimistic in structured settings. 
% A common approach in preference elicitation is to assume that preferences satisfy additional structure, under which elicitation can become easier. For instance, under single-peaked (resp. Euclidean) preferences, the query complexity of recovering a ranking drops from $\Theta(m \log m)$ to $O(m)$ (resp. $O(\log m)$)~\citep{conitzer2009, Jamieson_Nowak_deuclidean}. Similar improvements hold in other structured domains, such as preferences single-peaked on a tree~\citep{DeyMishra_2016_trees}, or on preferences profiles drawn from statistical cultures such as the Mallows model~\citep{peters_procaccia_2021}. 
% Motivated by these results, 
We ask here whether structural assumptions can simplify the hierarchy introduced in this paper. We focus on two classical models representing two distinct types of structure: the \emph{Plackett-Luce}  model~\citep{luce1959individual,plackett1975analysis}, a probabilistic latent-utility model, where each alternative has a strength $v_a > 0$, and rankings are  generated by sequential sampling without replacement, with probabilities proportional to these strengths; and \emph{single-peaked preferences}~\citep{black1948rationale,  moulin1980strategy}, where alternatives are ordered on a fixed one-dimensional axis and each voter's preferences are monotone away from their personal peak. We show that both models collapse the hierarchy to level $2$, that is, all higher-degree data can be recovered from degree~$2$ data. In particular, practitioners with reasons to believe that any of such structure holds can reduce elicitation to pairwise comparisons alone. The proof is included in Appendix \ref{app:mcl-lb-proof}.

%First, we show that if a preference profile is generated by a Plackett--Luce model with strength parameters $(v_a)_{a \in \calA}$, then the hierarchy collapses entirely to level~$2$. 
% We first consider the Plackett--Luce model. 

\begin{restatable}{proposition}{pl}\label{prop:pl}
Let $\calA$ be a set of alternatives. Consider $\pi_{\text{PL}}$ a preference profile following the Plackett-Luce model. Then, for any $S \subseteq \calA$ and $a \in S$, it follows $ p_{S}^{\pi_{\text{PL}}}(a) = {v_a}/{\sum_{x \in S} v_x}.$
%\begin{align*}
%    p_{S}^{\pi^{\text{PL}}}(a) = \frac{v_a}{\sum_{x \in S} v_x}.
%\end{align*}

Consider $\pi_{\text{SP}}$ a preference profile having in its support only preferences single-peaked over an axis $(s_1,...,s_m)$. Then, for any $S = \{s_{i_1},..., s_{i_k}\}\subseteq\calA$ sorted along the axis and $j \in \{1,...,k\}$, it follows,
\begin{align*}
p_S^{\pi_{\text{SP}}}(s_{i_j}) = p_{s_{i_j} s_{i_{j+1}}}^{\pi_{\text{SP}}}(s_{i_j}) - p_{s_{i_{j-1}} s_{i_j}}^{\pi_{\text{SP}}}(s_{j-1}), \text{ where } p_{s_{i_0}s_{i_1}}^{\pi_{\text{SP}}}(s_1) := 0 \text{ and } p_{s_{i_k}s_{i_{k+1}}}^{\pi_{\text{SP}}}(s_{k}) := 1.
% \\
% \\
% \begin{cases}
%   p_{s_{i_1} s_{i_2}}^{\pi^{\text{SP}}}(s_{i_1}) & j = 1,\\[3pt]
%   p_{s_{i_j} s_{i_{j+1}}}^{\pi^{\text{SP}}}(s_{i_j}) - p_{s_{i_{j-1}} s_{i_j}}^{\pi^{\text{SP}}}(s_{j-1}) & 1 < j < k,\\[3pt]
%   1 - p_{s_{i_{k-1}} s_{i_k}}^{\pi^{\text{SP}}}(s_{k-1}) & j = k.
% \end{cases}
\end{align*}
In particular, for each case, for any $k \geq 3$, the information of degree $k$ can be obtained as a function of the information of degree $2$.
\end{restatable}

Note that, although both structural assumptions \textit{collapse} the hierarchy to level~$2$, single-peakedness requires knowing the axis in advance, which is rarely the case in practice \citep{conitzer2009}. In contrast, in the Plackett--Luce model, the latent strengths $(v_a)_{a \in \calA}$ need not to be accessed or estimated, thus, eliciting pairwise comparisons is sufficient to recover all higher-degree quantities.

\section{Elicitation Protocols for the Plurality Matrix}\label{sec:elicitation}

This section is devoted to studying protocols to approximate the plurality matrix. In particular, we adopt a PAC approach, that is, to find $\varepsilon$-approximations of each matrix entry with probability $1-\delta$. 

\subsection{Elicitation Primitives and Protocols}\label{sec:primitives}

Estimating the information of degree $2$ of the plurality matrix trivially reduces to querying pairwise comparisons. For $|S| \geq 3$, estimating $p_S^\pi(a)$ is no longer direct, and the choice of elicitation primitive becomes nontrivial. We consider two primitives, described as sequences of pairwise comparisons.\footnote{Using pairwise comparisons as elicitation primitives allows us to compare protocols on their cognitive load, since pairwise comparisons are a standard proxy for cognitive load in elicitation studies \cite{conitzer2009}.}
%Designing elicitation primitives based on pairwise comparisons allows us to evaluate our protocols in terms of the cognitive load they impose on voters, as pairwise comparisons are typically used as a proxy for cognitive load in the elicitation literature \citep{conitzer2009}.}

\begin{definition}\label{def:elicitation_primitives}
Let $S = \{a_1, \ldots, a_{|S|}\} \subseteq \calA$ be a set of alternatives. 
\begin{itemize}[leftmargin = 0.45cm]\setlength\itemsep{0mm}
\vspace{-0.25cm}
    \item[1.] We define the \textbf{$S$-chain} as the primitive that performs $|S|\!-\!1$ sequential pairwise comparisons: compare $a_1$ and $a_2$, then iteratively compare the running winner with $a_j$ for $j \in \{3, \ldots, |S|\}$.
    \item[2.] We define the \textbf{$S$-ranking} as the primitive that elicits a full ranking of $S$ via pairwise comparisons.
\end{itemize}
\end{definition}
% \begin{definition}\label{def:s-chain}
% Given an ordering $(a_1, \ldots, a_{|S|})$ of $S$, the \textbf{$S$-chain} performs $|S|\!-\!1$ sequential pairwise comparisons: compare $a_1$ and $a_2$, then iteratively compare the running winner with $a_j$ for $j \in \{3, \ldots, |S|\}$.
% \end{definition}
% \begin{definition}\label{def:s-ranking}
% Given $S \subseteq \calA$, the \textbf{$S$-ranking} elicits a complete ranking of $S$ via pairwise comparisons.
% \end{definition}
Based on these elicitation primitives, we design two elicitation \textit{protocols}. For this, given a profile $\pi$, we denote $\calV^\pi_n$ a set of $n$ voters drawn independently from $\pi$. 

\begin{definition}\label{def:elicitation_protocols}
Given $N \leq n$ and $k \leq m$, a \textbf{$k$-chain (resp. $k$-ranking) protocol of length $N$} samples uniformly with replacement $N$ subsets $S_1,...,S_N$ each of size $k$, with $S_i\subseteq \calA$, and uniformly without replacement $N$ voters $\{v_1,...,v_N\}$ from $\calV^\pi_n$, and for each $i \in \{1,...,N\}$, asks voter $v_i$ to answer the $S_i$-chain (respectively, $S_i$-ranking). 
% \MT{To rephrase - $S$ are with replacements (so even writing $(S_1, ..., S_N)$ feels misleading}
% A protocol of length $N$ proceeds in $N$ steps: at each step, a size-$k$ subset $S$ is drawn uniformly from the $\!\binom{m}{k}\!$ candidates, and a fresh voter from $\calV^\pi$ answers via $S$-chain or $S$-ranking. We call these procedures \textbf{$k$-chain} and \textbf{$k$-ranking}; each observation is an unbiased sample from $\pi$.    
\end{definition}

\begin{remark}\label{remark:cognitive_load}
Protocols in \Cref{def:elicitation_protocols} differ in \textit{cognitive load} since an $S$-chain requires $|S|\!-\!1$ pairwise comparisons while an $S$-ranking requires $\Theta(|S| \log |S|)$ (e.g. when using merge sort).
\end{remark}
% We build protocols by repeating these primitives. Let $\calV = \{\sigma_1, \ldots, \sigma_n\}$ be $n$ voters drawn i.i.d.\ from $\pi$. A protocol of length $N \leq n$ proceeds in $N$ steps; at each step, a size-$k$ subset $S$ is drawn uniformly from the $\!\binom{m}{k}\!$ candidates, and a fresh voter answers via $S$-chain or $S$-ranking. We call these procedures \textbf{$k$-chain} and \textbf{$k$-ranking}; each observation is an unbiased sample from $\pi$.

With this in mind, we define two cost axes to assess our protocols.

\begin{definition}\label{def:budget_and_max_cognitive_load}
Given a protocol of length $N$, let $c_j$ be the number of pairwise comparisons performed at step $j$. We define the \textbf{budget} $B$ and the \textbf{maximum cognitive load} $\MCL$ of the protocol as:
% as the total number of pairwise comparisons performed on the $N$ steps, and the \textbf{maximum cognitive load} of the protocol as the cost of the heaviest single query, formally given 
\begin{align*}
    B := \sum\nolimits_{j=1}^{N}\! c_j \text{ and } \MCL := \max_{j\in\{1,...,N\}} c_j.
\end{align*}
% $B = \sum_{j=1}^{N}\! c_j$, that is, the total pairwise comparisons across the $N$ steps; the \textbf{maximum cognitive load} $\MCL = \max_j\! c_j$ is the cost of the heaviest single query.    
\end{definition}

The budget of a protocol is the total number of pairwise comparisons it requires, and its maximum cognitive load is the most demanding query it poses to a voter, ranging from a single pairwise comparison to a full ranking.
%The simplest query consists of a single pairwise comparison; at the other, the most demanding query requires a voter to report a full ranking over all alternatives. 
Given a population of voters, estimating the plurality matrix involves a fundamental trade-off between budget and maximum cognitive load. Intuitively, when voters can only be asked simple queries, i.e., involving few pairwise comparisons, a larger number of queries is required to obtain accurate estimates. Conversely, allowing more complex queries reduces their number at the cost of an increased cognitive burden per query.
Equivalently, this can be viewed as a tradeoff between the number of queried voters and the maximum cognitive load under a fixed budget.
% For $k$-chains we obtain that $\MCL = k\!-\!1$, while for $k$-rankings, $\MCL = \Theta(k \log k)$.

\subsection{Estimating the Plurality Matrix}

% We adopt a PAC framework: estimate every degree-$\ell$ entry $\splur{T}{a}$ of the plurality matrix to accuracy $\varepsilon$ with confidence $1 - \delta$. Two quantities drive the analysis.

Denote $Q_k$ the total number of non-zero plurality matrix entries of degree $k$ to be estimated. It follows that $Q_k = k \binom{m}{k}$. In order to estimate these values, we will look at the aggregated number of samples we require over the entries. 
%Denoting $T_k$ the total number of samples of plurality matrix entries of degree $k$ and 
Applying Hoeffding's inequality plus a union bound (see Appendix \ref{app:hoeffding-union-bound}), we obtain that $T_k = \ln(2 Q_k / \delta)/(2 \varepsilon^2)$ is the minimum number of samples required to obtain an $\varepsilon$-approximation with probability $1-\delta$ of all plurality matrix entries of degree $k$. 

With this in mind, we can compare our two elicitation protocols with respect to the number of samples they produce. An $S$-chain yields $|S|\!-\!1$ samples, one per degree, on the nested prefixes $\splurVar{a_1a_2}, \splurVar{a_1a_2a_3}, \ldots, \splurVar{a_1\ldots a_{|S|}}$. In particular, no extra information can be deduced from these values, as every time we ask a new pairwise comparison during an $S$-chain, the result is conditioned to the previous comparisons (see Appendix \ref{app:chaintransitivity} for a detailed discussion). Conversely, an $S$-ranking reveals the voter's preference restricted to $S$ which, by transitivity, determines the winner of every set $T \subseteq S$, yielding $\binom{|S|}{k}$ independent samples at each degree $k \leq |S|$, totaling $2^{|S|} - |S| - 1$ values per voter. Clearly, this connects with \Cref{remark:cognitive_load} regarding the cognitive load of each primitive. 

The following result formally states the previous discussion. Its proof and a more detailed analysis is included in Appendix \ref{thm:budget-tradeoff2}.

\begin{restatable}[]{theorem}{budgetTradeoffThm}
\label{thm:budget-tradeoff}
Let $\calA$ be a set of $m$ alternatives and $k\geq 2$ be fixed.
% , and take $\lambda \in [k-1,m-1]$. 
Suppose we want to estimate all $Q_k$ entries of degree $k$ of the plurality matrix up to accuracy $\varepsilon$ with probability at least $1\!-\!\delta$. Then, 
\begin{itemize}[leftmargin = 0.45cm]\setlength\itemsep{0mm}
\vspace{-0.25cm}
    \item[1.] A $k$-chain requires $N_{\text{chain}} := \binom{m}{k}\ln(2 Q_k / \delta)/(2 \varepsilon^2)$ voters and entails a maximum cognitive load of $\lambda\!=\!k-1$ per voter.
    %and a budget equal to $B_{\text{chain}} = (k-1)\cdot N_{\text{chain}}$.
    \item[2.] A $k$-ranking requires 
    $$N_{\text{rank}} := \binom{m}{k}\frac{\ln(2 Q_k / \delta)}{2 \varepsilon^2}\cdot \frac{1}{k\log{k}}$$ 
    voters and entails a maximum cognitive load $\lambda = k\log(k)$ per voter.
    %, and a budget equal to $B_{\text{rank}} = k\log(k)\cdot N_{\text{rank }}$.
\end{itemize}
% Fix a target degree $\ell \geq 2$ and an $\MCL$ budget $\lambda$ with $\ell - 1 \leq \lambda \leq m - 1$. To estimate every degree-$\ell$ entry of the plurality matrix to accuracy $\varepsilon$ with probability at least $1 - \delta$:
% \begin{enumerate}
% \item[\emph{(i)}] \emph{The $\ell$-chain} (feasible for any $\lambda \geq \ell - 1$) requires a population $N_{\mathrm{ch}} := \binom{m}{\ell}\,T_\ell$, achieves $\MCL = \ell - 1$, and incurs total budget $B_{\mathrm{ch}} = (\ell - 1)\,N_{\mathrm{ch}}$.
% \item[\emph{(ii)}] \emph{The $k(\lambda)$-ranking}, where $k(\lambda) := \max\{k \in [\ell, m] : \lceil \log_2(k!) \rceil \leq \lambda\}$ (feasible whenever $\lambda \geq \lceil \log_2(\ell!) \rceil$), requires a population $N_{\mathrm{rk}} := \binom{m}{\ell}\,T_\ell \,\big/\, \binom{k(\lambda)}{\ell}$, achieves $\MCL = \lceil \log_2(k(\lambda)!) \rceil \leq \lambda$, and incurs total budget $B_{\mathrm{rk}} = \lceil \log_2(k(\lambda)!) \rceil \cdot N_{\mathrm{rk}}$.
% \end{enumerate}
\end{restatable}

\Cref{thm:budget-tradeoff} characterizes the trade-off induced by our protocols. In particular, for fixed values of $\varepsilon$ and $\delta$, the range of entries of the plurality matrix that can be accurately estimated depends on the number of available voters. In applications such as online platforms, where the number of participants is typically large, designers may favor protocols with low cognitive load. Conversely, in settings with a limited number of voters who are willing to answer more demanding queries, eliciting rankings appears to be more appropriate.

Observe that, when estimating degree $k$ entries of the plurality matrix using rankings, the platform designer is not restricted to $k$-rankings: any $k'$-ranking with $k'\!>\!k$ can be used. Indeed, the transitivity of rankings allows one to extract information about entries of degree $k$ more efficiently, potentially reducing the number of required queries. However, increasing the size of the ranking also increases the maximum cognitive load on each voter. In particular, given a maximum cognitive load constraint $\lambda^*$, the largest ranking size $k^*$ that can be used is $k^* := \max\{ k \in \{1,...,m\} \mid \lceil \log_2(k!) \rceil \leq \lambda^*\}$.
%\begin{align*}
%    k^* := \max\{ k \in \{1,...,m\} \mid \lceil \log_2(k!) \rceil \leq \lambda^*\}.
%\end{align*}

Clearly, as already observed in \Cref{prop:hierarchy-strict}, the degree of the estimated entries of the plurality matrix will never be higher than the value $k$ of the chain or ranking used. This connects to the maximum cognitive load of the protocol, as stated in the following result proved in Appendix \ref{app:mcl-lb-proof}.

\begin{restatable}{proposition}{mcl}\label{prop:mcl-lb}
Any elicitation protocol estimating a plurality matrix entry of degree $k$ has a maximum cognitive load $\MCL$ of value at least $k - 1$. Thus, $k$-chains have the minimal cognitive load among all elicitation protocols for data of degree $k$. 
\end{restatable}

We conclude this section by illustrating the tradeoff between maximum cognitive load $\lambda$ and the required number of voters $N$ for our two protocols at different degrees, sampling from an impartial culture profile with $10$ alternatives. For each pair $(\lambda,N)$, \Cref{fig:budget-mcl} shows which of the two protocols is optimal in terms of budget. We observe that chains are optimal whenever voters present low maximal cognitive load values or when a large population of voters is available. For the rest of the regimes, thanks to inferring prefrences via transitivity, rankings become the optimal choice. 

\begin{figure}[h]
\centering
\includegraphics[width=0.8\linewidth]{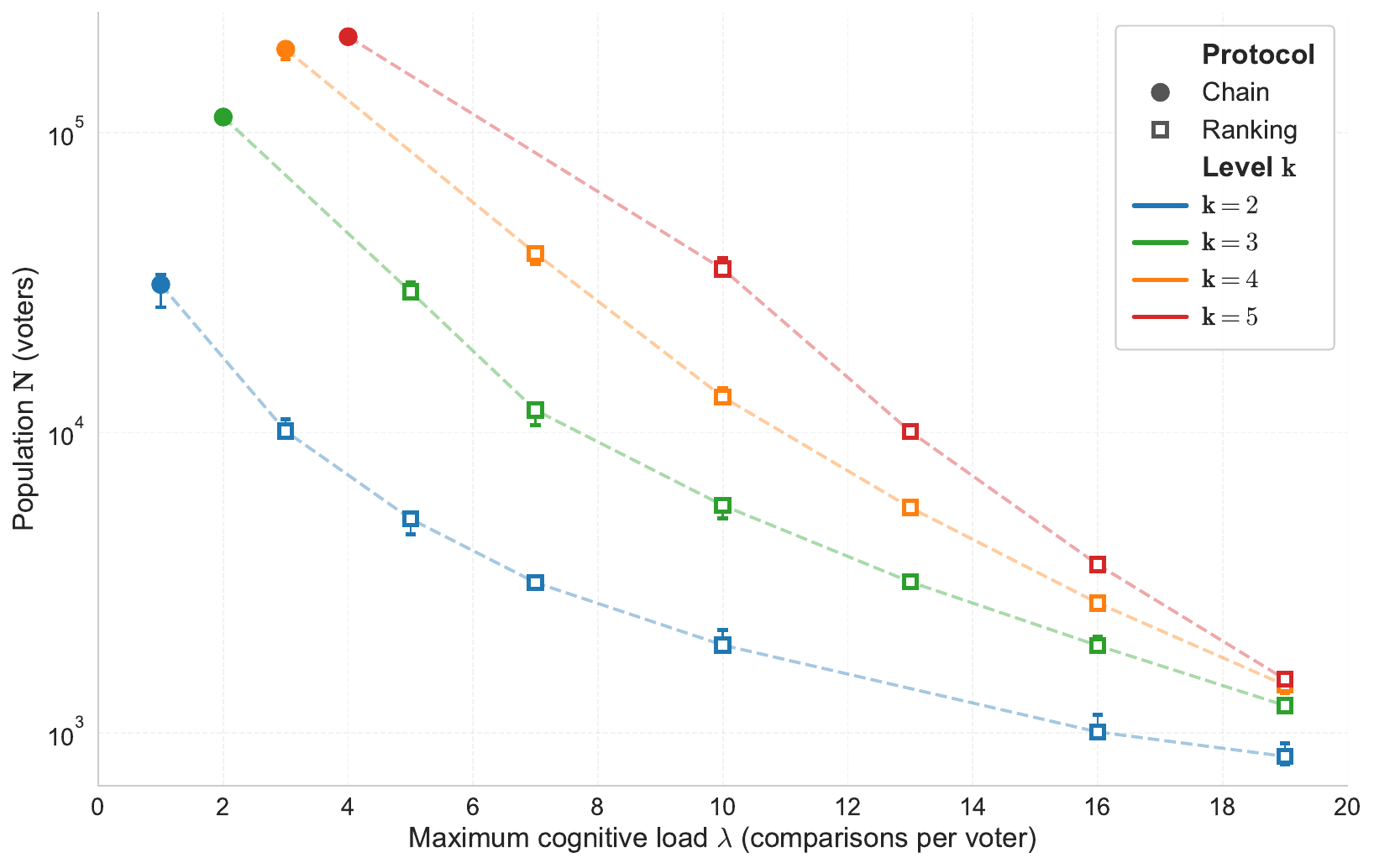}
\caption{Population $N$ (number of voters) required as a function of maximum cognitive load $\MCL$, across degrees $k \in \{2, 3, 4, 5\}$, on an impartial-culture profile with $m = 10$ alternatives and accuracy $\varepsilon = 0.05$. For each pair $(\lambda,N)$ the optimal protocol in terms of budget is highlighted. The intervals over the points represent the $5$th to the $95$th percentile.}
\label{fig:budget-mcl}
\end{figure}

% \Cref{fig:budget-mcl} visualizes both branches at $m = 10$. The leftmost (filled) point of each color is the $\ell$-chain, sitting at the minimum $\MCL = \ell - 1$ and the largest population; the open squares to its right are the $k$-rankings, $k$ growing from the smallest feasible value to $m$. Each branch is roughly linear on the log-population axis: each unit of $\MCL$ headroom cuts the required population by a roughly constant factor, with the rightmost ranking ($k = m$) requiring up to two orders of magnitude fewer voters than the chain endpoint at the same $\ell$. Raising the target degree $\ell$ shifts the entire curve upward, since $T_\ell$ and $\binom{m}{\ell}$ both grow with $\ell$. The two regimes sit at opposite ends of every front: low-$\MCL$ chains suit digital-democracy consultations where attention is scarcer than participant count, while high-$\MCL$ rankings suit expert-panel settings where deep effort per voter is acceptable. The formal derivation of the achievable region and its Pareto-optimal subset is given in \Cref{app:pareto_region}.% 
%%%%%%%%%%%%%%%%%%%%%%%%%%%%%%%%%%%%%%%%%%%%%%%%%%%%%%%%%%%%%%%%%%%%%%

%\MT{A very first draft of conclusion (far from being perfect, but we need to start from somewhere)}

\section{Conclusions}

To identify the disagreements of a population over a set of alternatives we introduce the plurality matrix, a stratified framework for organizing aggregated preference data according to subset sizes.
Thus, we can define the level of a disagreement measure as the smallest subset size required to express it. 
Our definitions allow us to chart an hierarchy of disagreement measures, with notable examples such as rank variance and divisiveness both lying at level $3$, showing, in particular, that neither of these measures can be recovered from pairwise comparisons alone. 
We further establish that the induced hierarchy is strict, by constructing preference profiles that coincide up to level $k$ but differ at level $k\!+\!1$. Moreover, for every level $k$, we exhibit a meaningful measure that inherently requires information from level $k$. 
Interestingly, structural assumptions on voters' preferences such as single-peakedness or the Plackett–Luce model collapse the hierarchy to level $2$, indicating that in rare albeit well-behaved settings disagreements can be fully captured through pairwise comparisons. 
Finally, we propose two elicitation protocols and analyze the tradeoff they induce between the cognitive load demanded to each voter and the population size.
%subset-size stratification of pairwise comparisons, and defined a disagreement measure's level as the smallest subset size expressing it. Rank variance and divisiveness sit at level $3$, separated by $\piic$ and $\piAN$. 
% The hierarchy is strict: for each $k$, there exist profiles that coincide up to level $k$ yet differ at $k+1$. 
%profiles exist that coincide up to level $k$ and differ at $k+1$. 
% More generally, the $k$-th central moment of an alternative's rank has a closed-form expression at level~$(k+1)$, so that higher moments such as skewness (level~$4$) and kurtosis (level~$5$) capture features invisible at lower levels. 
% At the same time, structured domains such as Plackett-Luce and single-peaked preferences collapse the hierarchy to level~$2$. 
% Finally, we propose two elicitation protocols and study their tradeoff between per-voter cognitive load and population size.
%Since divisiveness requires full rankings, we estimate the plurality matrix via $k$-chains and $k$-rankings, trading per-voter cognitive load against population, and characterize the achievable $(\MCL, B)$ region and its Pareto-optimal subset.}

\textbf{Limitations.} Our framework is not able to capture measures such as the $k$-Kemeny family, its derived polarization and diversity indices \citep{faliszewski2023}, or the rank-distance-based diversity measures \citep{hashemi2014measuring, karpov2017}, due to the nature of the event we measure within the entries of the plurality matrix. Additionally, our analysis relies on standard assumptions: voters are i.i.d., truthful, and noiseless, and cognitive load is uniformly measured in terms of the number of pairwise comparisons. Moreover, our empirical evaluation is based solely on complete preference reports. Extending our framework to settings with top-$k$ lists, approval ballots, partial orders, or Bayesian variants with priors over response noise and cognitive load, suggest interesting directions for future work.

\section*{Acknowledgments}

\noindent This work was supported by the ANR LabEx CIMI (grant ANR-11-LABX-0040) within the French State Programme “Investissements d’Avenir.”

Funded by the European Union. Views and opinions expressed are however those of the author(s) only and do not necessarily reflect those of the European Union or the European Research Council Executive Agency. Neither the European Union nor the granting authority can be held responsible for them. This work is supported by ERC grant 101166894 “Advancing Digital Democratic Innovation” (ADDI). This project was supported by the European Union LearnData, GA no. 101086712 a.k.a. 101086712-LearnDataHORIZON-WIDERA–2022-TALENTS–01 (https://cordis.europa.eu/project/id/101086712), IAST funding from the French National Research Agency (ANR) under grant ANR–17-EURE–0010 (Investissements d'Avenir program), and the European Lighthouse of AI for Sustainability [grant number 101120237-HORIZON-CL4–2022-HUMAN–02].
\vspace{-0.9cm}
\begin{figure}[h]
\centering
\includegraphics[scale = 0.16]{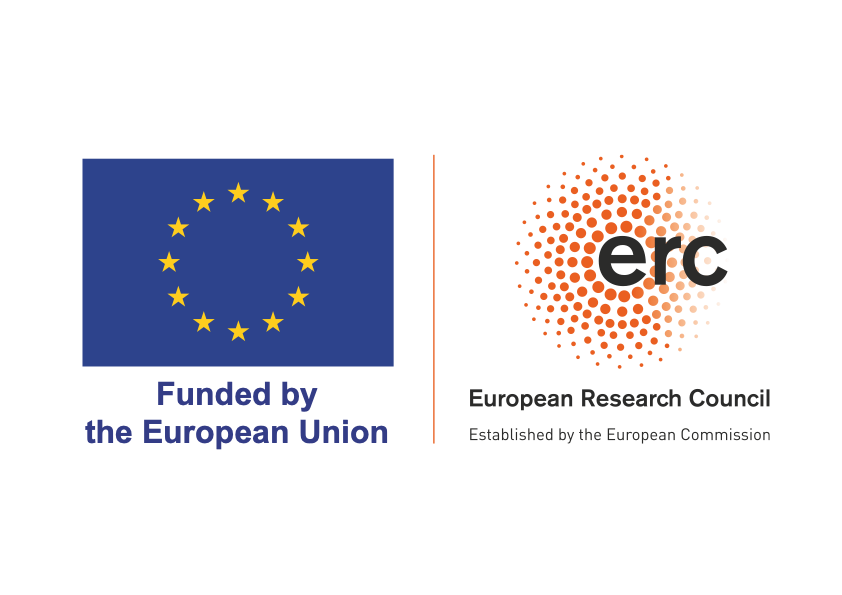}
\end{figure}
\vspace{-0.9cm}

\bibliographystyle{abbrv}
\bibliography{bibliography}

% \clearpage
\appendix
\addcontentsline{toc}{section}{Appendix}
\startcontents[appendix]
\vspace{2em}
\begin{center}
\Large\bfseries Appendix Table of Contents
\end{center}
\hrule
\vspace{1em}
\printcontents[appendix]{}{1}{\setcounter{tocdepth}{2}}
\vspace{1em}
\hrule
\vspace{1em}
%\rankPlurThm*

\section{Examples and omitted details}

\subsection{Bias in chain transitivity}\label{app:chaintransitivity}

%Using transitivity of voters' preferences to extract additional pairwise observations from a chain yields a biased estimator of pairwise proportions, even under uniform randomization of the chain ordering. We give an explicit three-alternative example.
No additional \emph{unbiased} observations can be inferred by transitivity from $k$-chains. Indeed, the probability that a given pair is resolved varies with the underlying ranking of each voter, inducing a selection bias. As a result, empirical frequencies computed from such inferred comparison do not, in general, estimate the true pairwise proportion. We illustrate this phenomenon on the following example. 

\begin{example}\label{ex:chaintransitivity}
Let $\calA = \{a, b, c\}$ and
\[
\pi \;=\; \tfrac{1}{2}\,\sigma_A \;+\; \tfrac{1}{2}\,\sigma_B,
\qquad
\sigma_A: a \succ c \succ b,
\qquad
\sigma_B: b \succ c \succ a,
\]
so that $p_{ca} = \Pr_\pi[c \succ a] = \tfrac{1}{2}$. Run a $3$-chain on $S = \{a, b, c\}$ with ordering $\tau$ drawn uniformly from the $6$ permutations, independently of the voter. For a voter $\sigma$, record
\[
Y_{ca}(\sigma, \tau) =
\begin{cases}
\ind[c \succ_\sigma a] & \text{if $\{a,c\}$ is compared directly or $w_3 \in \{a,c\}$ (transitive inference),} \\
\bot & \text{otherwise,}
\end{cases}
\]
and estimate $p_{ca}$ by $\hat p_{ca} = \Pr[Y_{ca} = 1 \mid Y_{ca} \neq \bot]$. Enumerating all $2 \times 6$ cases:

\begin{center}
\small\renewcommand{\arraystretch}{1.05}
\begin{tabular}{@{}l cc@{\hspace{2em}} cc@{}}
\toprule
 & \multicolumn{2}{c}{$\sigma_A$ ($w_3 = a$)} & \multicolumn{2}{c}{$\sigma_B$ ($w_3 = b$)} \\
\cmidrule(r{1em}){2-3}\cmidrule{4-5}
$\tau$ & resolved? & $Y_{ca}$ & resolved? & $Y_{ca}$ \\
\midrule
$(a,b,c)$ & yes (trans.) & $0$ & no & $\bot$ \\
$(a,c,b)$ & yes (direct) & $0$ & yes (direct) & $1$ \\
$(b,a,c)$ & yes (trans.) & $0$ & no & $\bot$ \\
$(b,c,a)$ & yes (trans.) & $0$ & no & $\bot$ \\
$(c,a,b)$ & yes (direct) & $0$ & yes (direct) & $1$ \\
$(c,b,a)$ & yes (trans.) & $0$ & no & $\bot$ \\
\midrule
rate & $6/6$ & all $0$ & $2/6$ & all $1$ \\
\bottomrule
\end{tabular}
\end{center}

Aggregating over the uniform voter-ordering distribution gives $\Pr[Y_{ca} = 0] = \tfrac{1}{2} \cdot \tfrac{6}{6} = \tfrac{1}{2}$, $\Pr[Y_{ca} = 1] = \tfrac{1}{2} \cdot \tfrac{2}{6} = \tfrac{1}{6}$, and $\Pr[Y_{ca} \neq \bot] = \tfrac{2}{3}$, so
\[
\hat p_{ca} \;=\; \frac{1/6}{2/3} \;=\; \tfrac{1}{4}, \qquad \text{bias} \;=\; \hat p_{ca} - p_{ca} \;=\; -\tfrac{1}{4}.
\]

The estimator's observation rate depends on the voter: a voter with $\text{top}_\sigma(S) \in \{a,c\}$ always resolves $\{a,c\}$ via transitivity ($6/6$ orderings), whereas a voter with $\text{top}_\sigma(S) \notin \{a,c\}$ resolves it only when $a$ and $c$ are placed at positions $1$–$2$ ($2/6$ orderings). Equal-mass voter types therefore produce unequal observation counts, and the conditional distribution of $Y_{ca}$ given a non-$\bot$ record differs from $\pi$'s marginal on $\{a,c\}$. The asymmetry is between voters, not between orderings, so uniform randomization of $\tau$ does not correct it. Only the nested prefix winners $w_2 = \text{top}_\sigma(\{x_1, x_2\})$ and $w_3 = \text{top}_\sigma(S)$ yield unbiased samples, since their reference subsets are fixed by the protocol independently of $\sigma$.

\end{example}

\subsection{Hoeffding union bound}\label{app:hoeffding-union-bound}

\Cref{thm:budget-tradeoff} states a per-quantity sample budget $T_\ell$ in terms of accuracy $\varepsilon$ and confidence $\delta$. We derive that bound here under i.i.d.\ sampling, and then refine the analysis under the without-replacement assumption that each voter responds at most once.

\paragraph{Hoeffding's inequality and the union bound.}
We seek $T_\ell$ small enough to be cheap and large enough to guarantee
\[
  \Pr\bigl[\,\forall (T, a) \text{ with } |T| = \ell \text{ and } a \in T :\, |\widehat{\splur{T}{a}} - \splur{T}{a}| \leq \varepsilon\,\bigr] \;\geq\; 1 - \delta.
\]
For each such pair, $\splur{T}{a}$ is a Bernoulli parameter, and $\widehat{\splur{T}{a}}$ is the empirical frequency of the event ``$a$ tops $T$'' across $T_\ell$ observations. Hoeffding's inequality gives, for each $(T, a)$ separately,
\[
  \Pr\bigl[\,|\widehat{\splur{T}{a}} - \splur{T}{a}| > \varepsilon\,\bigr] \;\leq\; 2\exp(-2 T_\ell \varepsilon^2).
\]
A union bound over the $Q_\ell := \ell \binom{m}{\ell}$ such quantities yields the simultaneous guarantee whenever
\begin{equation}\label{eq:hoeffding-Tl}
  T_\ell \;\geq\; \frac{1}{2\varepsilon^2}\,\ln\!\bigl(2 Q_\ell / \delta\bigr).
\end{equation}
We refer to $T_\ell$ as the per-quantity sample budget. Population requirements depend on how many of the $Q_\ell$ quantities each query updates; the conversion is given in \Cref{thm:budget-tradeoff}.

\paragraph{Without-replacement refinement.}
The Hoeffding bound assumes i.i.d.\ sampling with replacement. In our setting each voter responds at most once, so the samples for a fixed quantity $\splur{T}{a}$ are drawn without replacement from a finite population of size $n$. Serfling \cite{serfling1974} sharpens Hoeffding's inequality in this regime: writing $T$ for the sample size and $f^* := (T - 1)/n$ for the sampling fraction, the deviation exponent improves from $-2 T \varepsilon^2$ to $-2 T \varepsilon^2 / (1 - f^*)$. Substituting into the union bound and solving for the smallest sufficient sample size $T_\ell'$, one obtains
\[
  T_\ell' \;=\; T_\ell \cdot \frac{n + 1}{n + T_\ell},
\]
where $T_\ell$ on the right-hand side denotes the Hoeffding bound \eqref{eq:hoeffding-Tl}. The relative reduction is
\[
  1 - \frac{T_\ell'}{T_\ell} \;=\; \frac{T_\ell - 1}{n + T_\ell}.
\]
Whenever $n \geq \binom{m}{\ell}\,T_\ell = N_{\mathrm{chain}}$, the population threshold above which the chain protocol is feasible (\Cref{cor:pareto}\,(iii)), this reduction is bounded by
\[
  \frac{T_\ell - 1}{n + T_\ell} \;\leq\; \frac{1}{\binom{m}{\ell} + 1},
\]
which falls below $2\%$ as soon as $\binom{m}{\ell} \geq 50$ and decays as $\Theta(1/\binom{m}{\ell})$ thereafter. We therefore use the simpler Hoeffding bound \eqref{eq:hoeffding-Tl} throughout the paper; under without-replacement sampling, every stated sample size and population requirement may be tightened by the factor $(n + 1)/(n + T_\ell)$ at no cost to validity.

\subsection{Achievable Pareto region}\label{app:pareto_region}

The corollary below describes the achievable region of $(\MCL, B)$ pairs from \Cref{thm:budget-tradeoff}, identifies its Pareto-optimal subset, and translates the result into a protocol choice for a fixed population.

\begin{restatable}[]{corollary}{paretoCor}
\label{cor:pareto}
Fix degree $\ell \geq 2$. The achievable region of $(\MCL, B)$ pairs from \Cref{thm:budget-tradeoff} is the union of:
\begin{enumerate}
  \item[\emph{(i)}] At the low-$\MCL$ end, the chain point $(\MCL, B) = \bigl(\ell - 1,\,(\ell-1)\binom{m}{\ell} T_\ell\bigr)$. For $\ell \geq 3$, no ranking matches this $\MCL$, since $\lceil \log_2(\ell!) \rceil > \ell - 1$. For $\ell = 2$, the chain point coincides with the $k = 2$ ranking, both reducing to pairwise comparisons.
  \item[\emph{(ii)}] At the high-$\MCL$ end, the ranking curve indexed by $k \in \{\ell, \ldots, m\}$, with $\MCL = \lceil \log_2(k!) \rceil$ (monotone increasing in $k$) and $B = \binom{m}{\ell} T_\ell \cdot \lceil \log_2(k!) \rceil / \binom{k}{\ell}$ (weakly decreasing in $k$, strictly so for $\ell \geq 3$). For $\ell = 2$, $B$ is flat from $k = 2$ to $k = 3$, and the $k = 3$ ranking is therefore Pareto-dominated by the chain endpoint at $\MCL = 1$. The curve terminates at the budget-minimising $m$-ranking.
\end{enumerate}
The Pareto-optimal subset of this region consists of clause~\emph{(i)}'s chain point together with the rankings $\{k = \ell, \ldots, m\}$ for $\ell \geq 3$, and the rankings $\{k = 2\} \cup \{k = 4, \ldots, m\}$ for $\ell = 2$ (where $k = 2$ subsumes the chain point).
\begin{enumerate}
  \item[\emph{(iii)}] \emph{Protocol choice given $n$ voters.} Under the without-replacement assumption ($N \leq n$), the minimum-$\MCL$ feasible protocol is the $\ell$-chain when $n \geq \binom{m}{\ell}\,T_\ell$ (so that $N_{\mathrm{chain}} \leq n$); otherwise it is the $k^*$-ranking, where $k^*$ is the smallest $k \in \{\ell, \ldots, m\}$ satisfying $\binom{k}{\ell} \geq \binom{m}{\ell}\,T_\ell / n$.
\end{enumerate}
\end{restatable}

\begin{proof}
\emph{Clause (i): chain endpoint.} \Cref{thm:budget-tradeoff}\,(i) at $\lambda = \ell - 1$ gives $N_{\mathrm{chain}} = \binom{m}{\ell}\,T_\ell$ and $B = (\ell - 1)\binom{m}{\ell}\,T_\ell$. The $\MCL$-optimality among ranking protocols requires $\lceil \log_2(\ell!) \rceil > \ell - 1$ for $\ell \geq 3$. Direct verification: $\lceil \log_2 6 \rceil = 3 > 2$, $\lceil \log_2 24 \rceil = 5 > 3$, $\lceil \log_2 120 \rceil = 7 > 4$. For $\ell \geq 5$, Stirling's bound $\log_2(\ell!) \geq \ell \log_2(\ell/e)$ gives $\log_2(\ell!) > \ell - 1$, since $\log_2(\ell/e) > 1 - 1/\ell$ holds for $\ell \geq 5$. For $\ell = 2$, $\lceil \log_2(2!) \rceil = 1 = \ell - 1$, and the $\ell$-chain coincides with the $k = 2$ ranking, both reducing to pairwise comparisons.

\emph{Clause (ii): ranking branch.} For $k \in \{\ell, \ldots, m\}$, \Cref{thm:budget-tradeoff}\,(ii) with $\lambda = \lceil \log_2(k!) \rceil$ gives $\MCL = \lceil \log_2(k!) \rceil$ (monotone increasing in $k$) and $N_{\mathrm{rank}} = \binom{m}{\ell}\,T_\ell / \binom{k}{\ell}$ (monotone decreasing in $k$). Write $C_\ell := \binom{m}{\ell} T_\ell$, so $B(k) = C_\ell \cdot \lceil \log_2(k!) \rceil / \binom{k}{\ell}$. We establish the monotonicity of $B$ in $k$ by computing the step ratio
\[
  \frac{B(k+1)}{B(k)}
  \;=\;
  \frac{\lceil \log_2((k+1)!) \rceil}{\lceil \log_2(k!) \rceil}
  \cdot
  \frac{\binom{k}{\ell}}{\binom{k+1}{\ell}}
  \;=\;
  \frac{\lceil \log_2((k+1)!) \rceil}{\lceil \log_2(k!) \rceil}
  \cdot
  \frac{k+1-\ell}{k+1}.
\]
The condition $B(k+1) \leq B(k)$ is therefore equivalent to
\begin{equation}\label{eq:monotone-condition}
  \lceil \log_2((k+1)!) \rceil \cdot (k+1-\ell)
  \;\leq\;
  \lceil \log_2(k!) \rceil \cdot (k+1).
\end{equation}
Using the identity $\lceil \log_2((k+1)!) \rceil \leq \lceil \log_2(k!) \rceil + \lceil \log_2(k+1) \rceil$ and rearranging, a sufficient condition for \eqref{eq:monotone-condition} is
\begin{equation}\label{eq:monotone-sufficient}
  \lceil \log_2(k+1) \rceil \cdot (k+1-\ell) \;\leq\; \ell \cdot \lceil \log_2(k!) \rceil.
\end{equation}
For $\ell \geq 3$ and $k \geq \ell$, $\lceil \log_2(k!) \rceil \geq \lceil \log_2(\ell!) \rceil \geq \lceil \log_2 6 \rceil = 3$, so the right-hand side is at least $3\ell \geq 9$. The left-hand side is $\lceil \log_2(k+1) \rceil \cdot (k+1-\ell)$, bounded by $(\log_2(k+1) + 1)(k+1-\ell)$. A direct check on the boundary $k = \ell$ gives left-hand side $\lceil \log_2(\ell+1) \rceil$, right-hand side $\ell \cdot \lceil \log_2(\ell!) \rceil$, and the inequality is strict by inspection for every $\ell \in \{3, 4, \ldots, 20\}$ (and asymptotically for $\ell \to \infty$ since $\ell \log_2(\ell!) / \log_2(\ell+1) \to \infty$). For $k > \ell$, both sides grow but the right-hand side dominates: $\lceil \log_2(k!) \rceil$ grows by $\Theta(\log k)$ per step while the left-hand side increment is $\lceil \log_2(k+2) \rceil - \lceil \log_2(k+1) \rceil = O(1)$ per step on the logarithm and one unit per step on the $(k+1-\ell)$ factor. Hence \eqref{eq:monotone-sufficient} holds strictly for $\ell \geq 3$ and every $k \in \{\ell, \ldots, m-1\}$, giving strict decrease of $B$.

For $\ell = 2$ and $k = 2$: left-hand side of \eqref{eq:monotone-sufficient} is $\lceil \log_2 3 \rceil \cdot 1 = 2$, right-hand side is $2 \cdot 1 = 2$; equality holds, $B(3) = B(2)$. For $\ell = 2$ and $k \geq 3$: left-hand side $= \lceil \log_2(k+1) \rceil (k-1)$, right-hand side $= 2\lceil \log_2(k!)\rceil$; strict inequality holds (verified by the same growth argument). Hence for $\ell = 2$, $B$ is flat on $\{k = 2, k = 3\}$ and strictly decreasing for $k \geq 3$.

In both cases the $m$-ranking minimises $B$.

\emph{Clause (iii): protocol choice given $n$.}
If $n \geq \binom{m}{\ell}\,T_\ell$, branch (i) at $\lambda = \ell - 1$ is feasible; by clause (i) of this corollary, no ranking matches this $\MCL$ for $\ell \geq 3$, so the $\ell$-chain is the minimum-$\MCL$ feasible protocol. (For $\ell = 2$ the $\ell$-chain and the $k = 2$ ranking are the same protocol.)

If $n < \binom{m}{\ell}\,T_\ell$, branch (i) is infeasible at every $\lambda$ since $N_{\mathrm{chain}}$ is independent of $k$. Branch (ii) makes the $k$-ranking feasible iff $n \geq \binom{m}{\ell}\,T_\ell / \binom{k}{\ell}$, equivalently $\binom{k}{\ell} \geq \binom{m}{\ell}\,T_\ell / n$. Since $\binom{k}{\ell}$ and $\lceil \log_2(k!) \rceil$ are both monotone increasing in $k$, the minimum-$\MCL$ feasible ranking corresponds to the smallest $k$ satisfying the inequality, namely $k^*$ as defined.
\end{proof}

\section{Deferred proofs}\label{app:mcl-lb-proof}

\subsection{Proof of \Cref{prop:var-splur}}

\varsplur*

\begin{proof}
Write $X_b := \ind[a \succ b]$ for each $b \neq a$, so that $r_a = m - \sum_{b \neq a} X_b$.
Since variance is invariant under translations and sign flips, 
\begin{align*}
\Var_\pi(r_a)
  &\;=\; \Var_\pi\!\Bigl(\,\textstyle\sum_{b \neq a} X_b\,\Bigr)
  \;=\; \E_\pi\!\Bigl[\Bigl(\textstyle\sum_{b \neq a} X_b\Bigr)^{\!2}\,\Bigr]
        \,-\, \Bigl(\E_\pi\!\Bigl[\textstyle\sum_{b \neq a} X_b\Bigr]\Bigr)^{\!2} \\
  %&\;=\; \sum_{b \neq a} p_{ab}
  %      \,+\!\!\sum_{\substack{b, c \neq a \\ b \neq c}}\!\!\splur{\{a,b,c\}}{a}
  %      \,-\, \sum_{b, c \neq a} p_{ab}\,p_{ac} \\
  %&\;=\; \sum_{b \neq a} p_{ab}(1 - p_{ab})
  %      \,+\!\!\sum_{\substack{b, c \neq a \\ b \neq c}}\!\!
  %          \bigl(\,\splur{\{a,b,c\}}{a} - p_{ab}\,p_{ac}\,\bigr),
\end{align*}

We now compute each term separately. Expanding the first term,
\begin{align*}
    \E_\pi\!\Bigl[\Bigl(\textstyle\sum_{b \neq a} X_b\Bigr)^{\!2}\,\Bigr] 
    &= \sum_{b \neq a} \E_\pi[X_b^2] + \sum_{\substack{b, c \neq a \\ b \neq c}}\!\!\E_\pi[X_b X_c]
\end{align*}

Since $X_b \in \{0,1\}$, we have $X_b^2 = X_b$, hence 
\[ \E_\pi[X_b^2] = \E_\pi[X_b] = p_{ab}.\]

Moreover, for $b \neq c$,
\[ 
    X_b X_b = \ind[a \succ b \text{ and } a \succ c],
\]
so that 
\[
\E_\pi[X_b X_c] = \splur{abc}{a}. 
\]
Thus, 
\[
\E_\pi\!\Bigl[\Bigl(\textstyle\sum_{b \neq a} X_b\Bigr)^{\!2}\,\Bigr]  
= \sum\limits_{b \neq a} p_{ab} + +\!\!\sum_{\substack{b, c \neq a \\ b \neq c}}\!\!\splur{abc}{a}.
\]

Regarding the second term, 
\[
\Bigl(\E_\pi\!\Bigl[\textstyle\sum_{b \neq a} X_b\Bigr]\Bigr)^{\!2} = \sum\limits_{b,c \neq a} p_{ab}\,p_{ac}.
\]
Combining the both terms yields 
\[
\Var_\pi(r_a) = \sum_{b \neq a} p_{ab}  +\!\!\sum_{\substack{b, c \neq a \\ b \neq c}}\!\!\splur{abc}{a} - \sum\limits_{b,c \neq a} p_{ab}\,p_{ac}.
\]

Finally, separating diagonal and off-diagonal terms in the last sum gives 
\[
\Var_\pi(r_a)
  \;=\;
  \sum_{b \neq a} p_{ab}\,(1 - p_{ab})
  \;+\!\!
  \sum_{\substack{b, c \neq a \\ b \neq c}}
    \!\!\bigl(\,\splur{abc}{a} - p_{ab}\,p_{ac}\,\bigr).
\]
\end{proof}

\subsection{Proof of \Cref{prop:div-splur}}

\divsplur*

\begin{proof}
We recall that, by definition, 
\begin{align*}
&\Div_\pi(a) := \frac{1}{m-1}\sum\nolimits_{b \in\calA\setminus \{a\}}
    \bigl|\Bor_\pi\bigl(a;\calN^{a \succ b}_\pi\bigr)
         \;-\; \Bor_\pi\bigl(a;\calN^{b \succ a}_\pi\bigr)\bigr|.
         % \text{ where},\\
% &\Bor\bigl(a;\calN^{x \succ y}_\pi\bigr) := \sum\nolimits_{b \in \calA \setminus \{a\}} \mathbb{P}_\pi(a \succ b \mid x \succ y), \text{ for any } x,y \in \calA.
\end{align*}
with 
\begin{align*}
\Bor_\pi\bigl(a;\calN^{x \succ y}_\pi\bigr) := \sum\nolimits_{b \in \calA \setminus \{a\}} \mathbb{P}_\pi(a \succ b \mid x \succ y),
\end{align*}

Using the identity $r_a = 1 + \sum_{c \neq a}\ind[c \succ a]$, and by linearity of conditional expectation, we obtain
\begin{equation}\label{eq:cond-rank-split}
  \E_\pi[r_a \mid \calE]
  \;=\; 1 + \sum_{c \neq a} \Pr_\pi[c \succ a \mid \calE].
\end{equation}
Therefore, 
\[ 
\Bor_\pi\bigl(a;\calN^{x \succ y}_\pi\bigr) = \sum_{c \neq a} \Pr_\pi(a \succ c \mid x \succ y) = (m-1) - \sum_{c \neq a} \Pr_\pi(c \succ a \mid x \succ y) = m - \E_\pi[r_a \mid x \succ y].
\]
Finally, the divisiveness formula rewrites as
\[ 
\Div_\pi(a) = \frac{1}{m-1}\sum_{b \neq a} \bigl| \E_\pi[r_a \mid b \succ a] - \E_\pi[r_a \mid a \succ b]\bigr|.
\]
We now compute these conditional expectations. 

\textbf{Computing $\E_\pi[r_a \mid b \succ a]$:} By \eqref{eq:cond-rank-split}, we have 
\[ 
\E_\pi[r_a \mid b \succ a] = 1 + \sum_{c \neq a} \Pr_\pi[c \succ a \mid b \succ a]. 
\]

For $c = b$, we have $\Pr_\pi[b \succ a \mid b \succ a] = 1$.

For $c \neq a, b$, inclusion--exclusion gives
\[
  \Pr_\pi[b \succ a \text{ and } c \succ a]
  \;=\; 1 - \Pr_\pi[a \succ b \text{ or } a \succ c]
  \;=\; 1 - p_{ab} - p_{ac} + \splur{abc}{a},
\]
and dividing by $\Pr_\pi[b \succ a] = 1 - p_{ab}$ yields
\[
  \Pr_\pi[c \succ a \mid b \succ a]
  \;=\; \frac{1 - p_{ab} - p_{ac} + \splur{abc}{a}}{1 - p_{ab}}
  \;=\; 1 + \frac{\splur{abc}{a} - p_{ac}}{1 - p_{ab}}.
\]
Summing over all $m - 2$ terms $c \neq a,b$ and adding the $c = b$ term gives: 
\begin{align*}\label{eq:div-first-expectation}
  \E_\pi[r_a \mid b \succ a]
  &\;=\; 1 + 1 + (m - 2) + \frac{1}{1 - p_{ab}}\!\!\sum_{c \neq a, b}\!\!
        \bigl(\splur{abc}{a} - p_{ac}\bigr) \\
  &\;=\; m + \frac{1}{1 - p_{ab}}\!\!\sum_{c \neq a, b}\!\!
        \bigl(\splur{abc}{a} - p_{ac}\bigr).
\end{align*}

\textbf{Computing $\E_\pi[r_a \mid b \succ a]$:} By \eqref{eq:cond-rank-split}, we have 
\[ 
\E_\pi[r_a \mid a \succ b] = 1+ \sum_{c \neq a} \Pr_\pi[c \succ a \mid a \succ b].  
\]
For $c=b$, we have $\Pr_\pi[b \succ a \mid a \succ b] = 0$.

For $c \neq a, b$,
\[
  \Pr_\pi[c \succ a \mid a \succ b]
  \;=\; 1 - \Pr_\pi[a \succ c \mid a \succ b]
  \;=\; 1 - \frac{\Pr_\pi[a \succ b \text{ and } a \succ c]}{p_{ab}}
  \;=\; 1 - \frac{\splur{\{a,b,c\}}{a}}{p_{ab}}.
\]

Summing over all $(m-2)$ terms $c \neq a,b$ in \eqref{eq:cond-rank-split}  gives 
\begin{align*}\label{eq:div-second-expecatation}
  \E_\pi[r_a \mid a \succ b]
  \;=\; 1 + (m - 2) - \frac{1}{p_{ab}}\!\!\sum_{c \neq a, b}\!\!\splur{abc}{a}
  \;=\; m - 1 - \frac{1}{p_{ab}}\!\!\sum_{c \neq a, b}\!\!\splur{abc}{a}.
\end{align*}

Plugging these both expectation terms back to the divisiveness formula yields the desired result. %This final result involves the plurality quantities of degree at most~$3$, hence $\Div(a)$ is of level at most~$3$.
\end{proof}

\subsection{Proof of \Cref{prop:hierarchy-strict}}

\strictness*
\begin{proof}
Set $m = d + 2$ and fix a focal alternative $a \in \calA$.

Let $w = (w_1, \ldots, w_m)$ be a probability distribution on rank positions of $a$, and define $\pi_w$ as follows: the rank of $a$ is drawn from $w$, and the other $m-1$ alternatives are then placed uniformly at random among the remaining positions. 

We first compute $S$-plurality quantities involving $a$. Fix $S \ni a$ with $|S| = s \geq 2$. Conditioning on the position $r_a = j$, the event that $a$ is top within $S$ occurs exactly when all other $s-1$ alternatives of $S$ are placed below position $j$. Since these alternatives are distributed uniformly, this happens with probability 
\[
\frac{\binom{m-j}{s-1}}{{m-1}{s-1}}. 
\]
Taking expectation over $r_a$ yields 
\begin{equation}\label{eq:splur-pi-w}
  \splur{S}{a}(\pi_w)
  \;=\;
  \sum_{j=1}^{m} w_j\,
    \frac{\binom{m - j}{s - 1}}{\binom{m - 1}{s - 1}}.
\end{equation}

Thus, $\splur{S}{a}(\pi_w)$ is a linear functional of $w$ depending on $S$ only through its size $s$. For subsets $S \not\ni a$, symmetry among non-focal alternatives implies that their $S$-plurality quantities depend only on $s$, and therefore coincide across all profiles in the family $\{\pi_w\}$.

Now let $\delta = w - w' \in \R^m$. From \eqref{eq:splur-pi-w}, two profiles $\pi_w$ and $\pi_{w'}$ agree on all degrees $s \leq d$ if and only if 
\begin{equation}\label{eq:matching}
  \sum_{j=1}^{m} \delta_j \binom{m-j}{s-1} \;=\; 0
  \qquad \text{for } s = \{1, \ldots, d\},
\end{equation}
a linear system of $d$ equations in $m=d+2$ unknowns. Viewed as a function of~$j$, the coefficient vector of an equation given a fixed $s$ is a polynomial of degree $s-1$ for $s \in \{1, \ldots, d\}$, producing polynomials of distinct degrees. The coefficient matrix thus has rank~$d$ and the solution set of \eqref{eq:matching} has two free parameters.
The degree-$(d + 1)$ separation condition $\sum_j \delta_j \binom{m-j}{d} \neq 0$ is governed by a polynomial of degree~$d$, which cannot be a linear combination of polynomials of strictly lower degree; it is therefore not implied by~\eqref{eq:matching}, and there exists $\delta^*$ satisfying~\eqref{eq:matching} while violating the separation equality.

Take $w$ uniform on $\{1, \ldots, m\}$ and set $w' := w + t\,\delta^*$ for $|t|$ small.
Since $\delta^*$ satisfies~\eqref{eq:matching} at $s = 1$, $\sum_j \delta^*_j = 0$, so $w'$ still sums to~$1$; since $w$ has strictly positive entries, $w'$ stays non-negative.
By construction, $\pi_w$ and $\pi_{w'}$ agree at every degree $\leq d$ and differ at degree $d + 1$.

\emph{Witness at $d = 3$, $m = 5$.}
Take $w = \bigl(\tfrac{1}{5}, \tfrac{1}{5}, \tfrac{1}{5}, \tfrac{1}{5}, \tfrac{1}{5}\bigr)$ and $w' = \bigl(\tfrac{3}{20}, \tfrac{7}{20}, \tfrac{1}{20}, \tfrac{5}{20}, \tfrac{4}{20}\bigr)$, differing by $\delta^* = \tfrac{1}{20}\cdot(-1, 3, -3, 1, 0)$.
Substituting in~\eqref{eq:splur-pi-w} gives $\splur{S}{a}(\pi_w) = \splur{S}{a}(\pi_{w'}) = 1/|S|$ for every $|S| \in \{2, 3\}$; at $|S| = 4$ the values are $1/4$ and $19/80$.
\end{proof}

\subsection{Proof of \Cref{thm:moment-level} and consequences}

\begin{definition}[Aggregate plurality]\label{def:agg-plur}
For $a \in \calA$ and $s \geq 0$,
\[
  P_s(a)
  \;:=\;
  \sum_{\substack{T \subseteq \calA \setminus \{a\} \\ |T| = s}}
  \splur{T \cup \{a\}}{a},
  \qquad
  P_0(a) := 1.
\]
\end{definition}

\begin{restatable}{lemma}{binmoment}
\label{lem:binomial-moment}
For every $a \in \calA$ and every $s \geq 0$,
$\E_\pi\bigl[\tbinom{W_a}{s}\bigr] = P_s(a)$.
\end{restatable}

\begin{proof}
For $T \subseteq \calA \setminus \{a\}$, the event
``$a \succ b$ for every $b \in T$'' has probability
$\splur{T \cup \{a\}}{a}$ and indicator
$\prod_{b \in T} \ind[a \succ b]$.
Summing over $|T| = s$ and taking expectations,
\(
  P_s(a)
  = \E_\pi\!\bigl[\sum_{|T| = s} \prod_{b \in T}\ind[a \succ b]\bigr]
  = \E_\pi\!\bigl[\tbinom{W_a}{s}\bigr],
\)
where the second equality holds because
$\sum_{|T|=s} \prod_{b \in T} \ind[a \succ b]$ counts the size-$s$
subsets of $\{b : a \succ b\}$.
\end{proof}

\momentlevel*
\begin{proof}
Write $X_b := \ind[a \succ b]$ for $b \neq a$, so
$r_a = m - \sum_{b \neq a} X_b$ and
$r_a - \mu_a = -\sum_{b \neq a}(X_b - p_{ab})$.
Set $B := \Bor(a)$ and
$g(X) := \bigl(\sum_{b \neq a}(X_b - p_{ab})\bigr)^k$, giving
$M_k(a) = (-1)^k \E_\pi[g(X)]$.
 
Since each $X_b \in \{0,1\}$ and $g$ has total degree at most $k$ in
the $X_b$'s (it is the $k$-th power of a linear form), $g$ reduces
on~$\{0,1\}^{m-1}$ to a multilinear polynomial
$g(X) = \sum_{|T| \leq k} \gamma(T) \prod_{b \in T} X_b$
supported on subsets $T \subseteq \calA \setminus \{a\}$ of size at
most~$k$.
Taking expectations,
$\E_\pi[g(X)] = \sum_T \gamma(T)\,\splur{T \cup \{a\}}{a}$
(using \Cref{lem:binomial-moment}).
 
By M\"obius inversion on the subset lattice,
$\gamma(T) = \sum_{U \subseteq T}(-1)^{|T|-|U|}\,g(e_U)$,
where $e_U \in \{0,1\}^{m-1}$ is the indicator vector of~$U$.
Evaluation at $X = e_U$ gives $g(e_U) = (|U| - B)^k$, which depends
on~$U$ only through~$|U|$, so the Möbius formula collapses to
$\gamma(T) = \sum_{j=0}^{|T|}(-1)^{|T|-j}\binom{|T|}{j}(j - B)^k
 = c_{|T|}(a)$
with $c_s(a)$ as stated in the theorem.
Collecting by $s = |T|$, $0 \leq s \leq k$, yields
$\E_\pi[g(X)] = \sum_{s=0}^{k} c_s(a)\,P_s(a)$, and multiplying by
$(-1)^k$ gives~\eqref{thm:moment-level}.
The right side depends on~$a$ only through $B = P_1(a)$ and on
$\pi$ only through $P_0(a), \ldots, P_k(a)$, hence on data of
degree at most $k + 1$.
\end{proof}

\begin{restatable}{corollary}{momentslevelk} \label{cor:moments-exact}
For every $k \geq 2$, the $k$-th central moment $M_k(a)$ has level
exactly $k + 1$.
\end{restatable}

\begin{proof}
The upper bound is \Cref{thm:moment-level}.
For the lower bound, apply \Cref{prop:hierarchy-strict} with
$d = k$: the witness profiles $\pi, \pi'$ agree at every degree
$\leq k$, so $P_0(a), P_1(a), \ldots, P_{k-1}(a)$ coincide on the
two profiles (each is determined by degree-$\leq k$ data), and in
particular $\Bor(a) = P_1(a)$ matches, making
$c_0(a), \ldots, c_k(a)$ identical on both profiles.
Only $P_k(a)$ differs (it depends on degree-$(k + 1)$ data).
Since $c_k(a) = k! \neq 0$ (the $k$-th forward difference of a monic
degree-$k$ polynomial), \eqref{thm:moment-level} gives
$M_k(a)(\pi) - M_k(a)(\pi') = (-1)^k\, k! \cdot
(P_k(a)(\pi) - P_k(a)(\pi')) \neq 0$.
So no function of degree-$\leq k$ data agrees with $M_k$ on these
two profiles.
\end{proof}

\subsection{Proof of \Cref{prop:pl}}
\pl*

\begin{proof}
Under Plackett--Luce, the probability that a voter ranks $a$ first among $S$ is obtained by the sequential elimination rule: at each step, the next-ranked alternative is drawn with probability proportional to its strenght.
Luce's Choice Axiom gives $\splur{S}{a} = v_a / \sum_{x \in S} v_x$ directly.
Since $p_{ab} = v_a/(v_a + v_b)$, the ratio $v_a/v_b = p_{ab}/(1-p_{ab})$ is determined by pairwise data.
Fixing any reference alternative~$r$, we recover $v_a/v_r$ for all~$a$, and then
\[
\splur{S}{a} = \frac{v_a/v_r}{\sum_{x \in S} v_x/v_r}
         = \frac{p_{ar}/(1-p_{ar})}{\sum_{x \in S} p_{xr}/(1-p_{xr})},
\]
which is a function of pairwise proportions alone.

Regarding the formula for single-peaked preferences, we prove it by induction on $|S|$.

\emph{Base case ($|S| = 2$).}
For $S = \{s_i, s_j\}$ with $i < j$: $\splur{S}{s_i} = p_{s_i s_j}$ and $\splur{S}{s_j} = 1 - p_{s_i s_j}$, which matches~\eqref{prop:pl}.

\emph{Key lemma ($|S| = 3$).}
For $a < b < c$ on the axis, single-peakedness implies that $b$ is never ranked last among $\{a, b, c\}$: every voter's peak is some element~$x$, and if $x \leq b$ then $b \succ c$ (closer to peak), while if $x \geq b$ then $b \succ a$.
In either case $b$ is not last.
Therefore $\Pr(a \succ b \text{ and } c \succ b) = 0$.

The three plurality scores sum to~$1$ and the ``$b$ is last'' event has probability zero, so:
\begin{align*}
\splur{\{a,b,c\}}{a} &= \Pr(a \succ b \text{ and } a \succ c) = \Pr(a \succ b) - \Pr(a \succ b \text{ and } c \succ a) \\
  &= p_{ab} - \Pr(c \succ a \succ b)
  = p_{ab} - \bigl(p_{ab} - \splur{\{a,b,c\}}{a}\bigr).
\end{align*}

Said differently: the events ``$a$ first'', ``$b$ first'', ``$c$ first'' partition the probability space.
The event ``$b$ first'' decomposes as $\{b \succ a\} \cap \{b \succ c\}$.
Since ``$b$ last'' $= \{a \succ b\} \cap \{c \succ b\}$ has probability zero, we have :
\begin{align*}
\splur{\{a,b,c\}}{b} &= \Pr(b \succ a \text{ and } b \succ c) \\
  &= 1 - \Pr(a \succ b) - \Pr(c \succ b) + \underbrace{\Pr(a \succ b \text{ and } c \succ b)}_{=\,0} \\
  &= 1 - p_{ab} - (1 - p_{bc}) = p_{bc} - p_{ab}.
\end{align*}
Then $\splur{\{a,b,c\}}{a} = p_{ab}$ and $\splur{\{a,b,c\}}{c} = 1 - p_{bc}$ follow from $ \splur{\{a,b,c\}}{a} +  \splur{\{a,b,c\}}{b} +  \splur{\{a,b,c\}}{c} = 1$ and $ \splur{\{a,b,c\}}{a} +  \splur{\{a,b,c\}}{b} = p_{ab} + (p_{bc} - p_{ab}) = p_{bc} = 1 -  \splur{\{a,b,c\}}{c}$.

\emph{Inductive step.}
Suppose the formula holds for all subsets of size $k - 1$.
Let $S = \{s_{i_1} < \cdots < s_{i_k}\}$ with $k \geq 4$.
For the interior element $s_{i_j}$ with $1 < j < k$:
\[
\splur{S}{s_{i_j}} = \Pr\bigl(s_{i_j} \succ s_{i_{j-1}} \text{ and } s_{i_j} \succ s_{i_{j+1}} \text{ and } s_{i_j} \succ \text{all others in } S\bigr).
\]
Single-peakedness on the axis implies two monotonicity properties for any voter:
\begin{enumerate}
\item[(i)] if $a < b < c$ on the axis and $a \succ b$, then $a \succ c$;
\item[(ii)] if $a < b < c$ on the axis and $c \succ b$, then $c \succ a$.
\end{enumerate}
By~(i), $s_{i_j} \succ s_{i_{j+1}}$ implies $s_{i_j} \succ s_{i_\ell}$ for all $\ell > j+1$;
by~(ii), $s_{i_j} \succ s_{i_{j-1}}$ implies $s_{i_j} \succ s_{i_\ell}$ for all $\ell < j-1$.
Therefore $s_{i_j}$ is first in $S$ if and only if $s_{i_j}$ is first in the three-element subset $\{s_{i_{j-1}}, s_{i_j}, s_{i_{j+1}}\}$.
By the base case, this probability is $p_{s_{i_j} s_{i_{j+1}}} - p_{s_{i_{j-1}} s_{i_j}}$.

For the boundary elements: $s_{i_1}$ is first in $S$ iff $s_{i_1} \succ s_{i_2}$ (by~(i), this implies $s_{i_1} \succ s_{i_\ell}$ for all $\ell > 2$), so $\splur{S}{s_{i_1}} = p_{s_{i_1} s_{i_2}}$.
Similarly, $s_{i_k}$ is first iff $s_{i_k} \succ s_{i_{k-1}}$, giving $\splur{S}{s_{i_k}} = 1 - p_{s_{i_{k-1}} s_{i_k}}$.
\end{proof}

\subsection{Proof of \Cref{prop:mcl-lb}}
\mcl*

\begin{proof}
Fix a subset $S$ with $|S| = k$ and an alternative $a \in S$. Each pairwise comparison between elements of $S$ creates an oriented link between two alternatives, and transitivity allows information to propagate along chains of such links. 

To determine whether $a$ is the top-ranked element of $S$, all alternatives of $S$ must be reachable from $a$ through such a chain of links; otherwise, there exists some $b \in S \setminus \{a\}$ whose relation to $a$ can not be resolved. 

Any structure connecting all $k$ elements must contain at least $k-1$ links. Therefore, any protocol must satisfy $\MCL \geq k - 1$.
\end{proof}

\begin{comment}
Both branches reduce to the same Hoeffding union bound, which we record once. Suppose for each pair $(T, a)$ with $|T| = \ell$ and $a \in T$ we have $T_\ell$ independent $\text{Ber}(\splur{T}{a})$ samples. By Hoeffding's inequality for $[0,1]$-valued random variables, the empirical mean $\widehat{\splur{T}{a}}$ satisfies
\[
  \Pr\!\bigl[\,|\widehat{\splur{T}{a}} - \splur{T}{a}| > \varepsilon\,\bigr]
  \;\leq\; 2\,\exp\!\bigl(-2\,T_\ell\,\varepsilon^2\bigr)
  \;=\; \delta / Q_\ell,
\]
by definition of $T_\ell$. A union bound over the $Q_\ell$ pairs yields
\(
  \Pr\bigl[\exists (T, a) :\, |\widehat{\splur{T}{a}} - \splur{T}{a}| > \varepsilon\bigr] \leq \delta.
\)
It remains to show that each branch delivers the required samples within its stated population.  
\end{comment}

\subsection{Detailed derivation of \Cref{thm:budget-tradeoff}}
\begin{theorem}
\label{thm:budget-tradeoff2}
Fix a target degree $\ell \geq 2$ and an $\MCL$ budget $\lambda$ with $\ell - 1 \leq \lambda \leq m - 1$. %Set
%\[
%  N_{\mathrm{chain}} \;:=\; \binom{m}{\ell}\,T_\ell, \qquad
%  N_{\mathrm{rank}} \;:=\; \binom{m}{\ell}\,T_\ell \,\big/\, \binom{k(\lambda)}{\ell},
%\]
%where $k(\lambda) := \max\{k \in \{\ell, \ldots, m\} : \lceil \log_2(k!) \rceil \leq \lambda\}$, defined whenever $\lambda \geq \lceil \log_2(\ell!) \rceil$. 
To estimate every degree-$\ell$ entry of the plurality matrix to accuracy $\varepsilon$ with probability at least $1 - \delta$:
\begin{enumerate}[leftmargin = *]
\item[\emph{(i)}] \emph{Chains.} 
Set $k = \lambda + 1$. A $k$-chain requires $N_{\mathrm{chain}} := \binom{m}{\ell} T_\ell$
queries, %(assuming $n \ge N_{\mathrm{chain}}$ distinct voters),
with total budget $B_{\mathrm{chain}} = N_{\mathrm{chain}} \cdot \lambda$.%, independent of $k$.
%Setting $k = \lambda + 1$, the $k$-chain run for $N \geq N_{\mathrm{chain}}$ steps succeeds, provided the population accommodates this many distinct voters ($n \geq N$), at total budget $B_{\mathrm{chain}} = N \cdot \lambda$. The bound is independent of $k$.
\item[\emph{(ii)}] \emph{Rankings.} %If $\lambda \geq \lceil \log_2(\ell!) \rceil$, the $k(\lambda)$-ranking run for $N \geq (1+o(1))\,N_{\mathrm{rank}}$ steps succeeds (provided $n \geq N$), at total budget $B_{\mathrm{rank}} = N \cdot \lceil \log_2(k(\lambda)!) \rceil \leq N \cdot \lambda$. If $\lambda < \lceil \log_2(\ell!) \rceil$, no ranking is feasible.
If $\lambda \ge \lceil \log_2(\ell!) \rceil$, define
$k(\lambda) := \max\{k \in [\ell,m] : \lceil \log_2(k!) \rceil \le \lambda\}$.
A $k(\lambda)$-ranking requires $N_{\mathrm{rank}} := \binom{m}{\ell} T_\ell \big/ \binom{k(\lambda)}{\ell}$ queries, %(assuming $n \ge N_{\mathrm{rank}}$),
with total budget $B_{\mathrm{rank}} = N_{\mathrm{rank}} \cdot \lceil \log_2(k(\lambda)!) \rceil \le N_{\mathrm{rank}} \lambda$.

If $\lambda < \lceil \log_2(\ell!) \rceil$, no ranking protocol is feasible.
\end{enumerate}
\end{theorem}

\begin{proof}
\medskip

\noindent \emph{Proof of (i): chain branch.}
The k-chain protocol admits any schedule that assigns each size-$\ell$ subset $T \subseteq \calA$ exactly $T_\ell$ chain steps; we take the simplest such allocation, which uses $N_{\mathrm{chain}} = \binom{m}{\ell}\,T_\ell$ steps. For $N > N_{\mathrm{chain}}$, allocate the first $N_{\mathrm{chain}}$ steps as below; the remaining $N - N_{\mathrm{chain}}$ steps may be used arbitrarily without affecting the accuracy guarantee.

For each step assigned to $T$, choose a size-$k$ superset $S \supseteq T$ (e.g., by adding $k - \ell$ filler alternatives) and an ordering $\tau = (a_1, \ldots, a_k)$ of $S$ in which the first $\ell$ positions form a permutation of $T$. Run the $S$-chain with ordering $\tau$ on a fresh voter $\sigma \in \calV$ not previously queried and record the prefix-$\ell$ winner $w_\ell$.

By induction on the chain rounds, the winner of the prefix $\{a_1, \ldots, a_j\}$ after round $j-1$ is the top of $\{a_1, \ldots, a_j\}$ in $\sigma$'s restriction to that prefix; this is exactly transitivity of $\sigma$. In particular, $w_\ell$ is the top of $T$ in $\sigma|_T$, so for each $a \in T$,
\[
  \Pr_\sigma[w_\ell = a] \;=\; \splur{T}{a}.
\]
The single observation $w_\ell$ thus yields $\ell$ correlated indicators $\{\ind[w_\ell = a]\}_{a \in T}$, each marginally $\text{Ber}(\splur{T}{a})$. Across the $T_\ell$ steps assigned to $T$, voters are drawn independently from $\pi$, so for each fixed $a \in T$ the $T_\ell$ indicators are i.i.d. The pairwise correlation among indicators sharing a step is irrelevant: Hoeffding applies separately to each pair $(T, a)$, and the union bound over $Q_\ell$ pairs closes the argument.

For the cost, each $k$-chain step uses exactly $k - 1 = \lambda$ pairwise comparisons, hence $\MCL = \lambda$ and $B_{\mathrm{chain}} = N \cdot \lambda$.

\medskip

\noindent \emph{Proof of (ii): ranking branch.}
Set $k := k(\lambda)$. The k-ranking protocol admits any schedule that produces a sequence $S_1, \ldots, S_N$ of size-$k$ subsets such that every size-$\ell$ subset $T \subseteq \calA$ is contained in at least $T_\ell$ of them. Counting incidences in two ways, this requires
\[
  N \binom{k}{\ell} \;\geq\; \binom{m}{\ell}\,T_\ell,
\]
which is exactly $N \geq N_{\mathrm{rank}}$. For $N$ above this bound by a $(1+o(1))$ factor, uniform random sampling realizes the required covering count with high probability by a Chernoff bound.

For each step $i$, elicit a $k$-ranking of $S_i$ from a fresh voter $\sigma_i \in \calV$ not previously queried. The ranking discloses $\sigma_i$'s restriction to $S_i$, and by transitivity this restriction determines the top of every $T \subseteq S_i$. For each pair $(T, a)$ with $T \subseteq S_i$ and $a \in T$, the indicator $\ind[a \text{ tops } T \text{ in } \sigma_i|_T]$ is therefore a $\text{Ber}(\splur{T}{a})$ sample. Voters across steps are independent, so for any fixed pair $(T, a)$ the samples drawn from the (at least $T_\ell$) steps with $T \subseteq S_i$ are i.i.d. Hoeffding plus the union bound close the argument as in branch (i).

For the cost, each ranking step requires $\lceil \log_2(k!) \rceil$ pairwise comparisons. The lower bound is information-theoretic: sorting $k$ items by binary comparisons demands at least $\lceil \log_2(k!) \rceil$ queries, since each query reveals one bit and $k!$ orderings must be distinguishable. The upper bound is achieved up to lower-order additive terms by Ford-Johnson merge-insertion sort, which lies within $\log_2(k!) + O(1)$ for small $k$ and within $(1+o(1)) \log_2(k!)$ asymptotically. By definition of $k(\lambda)$, $\lceil \log_2(k(\lambda)!) \rceil \leq \lambda$, so the per-voter cost respects the $\MCL$ constraint, and $$B_{\mathrm{rank}} = N \cdot \lceil \log_2(k(\lambda)!)\rceil \leq N \cdot \lambda.$$
\end{proof}

\section{Additional Experiments}

\subsection{Synthetic Data}\label{sec:synthetic_data}

We provide an informal explanation for the empirical observation in \Cref{fig:synthetic} that several structured preference models (single-peaked, Plackett--Luce and Euclidean) yield unimodal rank distributions for most alternatives/ 

\paragraph{A common intuition}
Across these models, a similar phenomenon is observed: alternatives that occupy ``central'' positions in the underlying structure are more likely to be ranked highly across voters, while ``extreme'' alternatives are less frequently top-ranked. This induces near-unimodal distributions of ranks. We will now specify how the notions of ``central'' and ``extreme'' translates under each of the structures. 

\paragraph{Single-peaked preferences (Walsh distribution).} In the single-peaked model, alternatives are ordered along an axis $c_ 1 < c_2 < \hdots < c_m$. Each voter has their most preferred candidate (peak) somewhere on the axis, and their preference decrease as we move away from the peak towards each of the extremes. 

Given such an axes, there are two classical ways to sample rankings:
\begin{itemize}
    \item[(i)] \textbf{the Conitzer distribution}, which samples the peak uniformly and then completes the ranking, and
    \item[(ii)] \textbf{the Walsh distribution} (used in our experiments), which samples uniformly from all rankings consistent with the axis.
\end{itemize}
Under the Walsh distribution, the number of rankings compatible with a given peak depends strongly on its position: while there is a unique preference with $c_1$ (resp. $c_m$) as peak ($c_1 \succ c_2 \succ \hdots \succ c_m$, resp. $c_m \succ c_{m-2} \succ \hdots \succ c_1$), there are exponentially many compatible rankings with the peak corresponding to middle alternative of the axis. As a result, peaks -- and more generally high-ranked alternatives-- are overwhelmingly concentrated around the middle of the axis. This induces unimodal rank distribution for most alternatives. 

\paragraph{Euclidean preferences.} In the $k$-Euclidean model, both voters and alternatives are embedded in $\mathbb{R}^k$, and each voter ranks alternatives by increasing distance from their own position. When both voters and alternatives are drawn from a roughly symmetric distribution (e.g., uniformly over a bounded region), alternatives near the geometric center of the space are, on average, closer to more voters. Consequently, central alternatives are more likely to be top-ranked, whereas peripheral alternatives tend to appear at lower positions. This again creates a concentration of probability mass in the upper ranks, leading to unimodal rank distribution. 

\paragraph{Plackett--Luce model.} In the Plackett--Luce model, each alternative $a$ is assigned a strength parameter $v_a$, and rankings are generated sequentially: the top alternative is drawn with probability proportional to these strengths, then removed, and the process repeats on the remaining set. When strengths are heterogeneous (as in our experiments), a small set of high-strength alternatives dominates top positions in the ranking, while weaker alternatives are pushed toward lower ranks. This induces a strong concentration of probability mass in high ranks for a few alternatives, again yielding unimodal rank distribution.

\subsection{Experiment on election data}\label{sec:real_data_experiment}

We now ask which regions of the moment plane are populated by real elections. We use three datasets covering different scales and electoral settings, drawn from PrefLib~\citep{boehmer2024guide} and from a French presidential election survey~\citep{voterautrement}. For consistency with the synthetic experiment, each dataset is restricted to voters who provided complete rankings.

\begin{figure}[h]
  \centering
  \includegraphics[width=\linewidth]{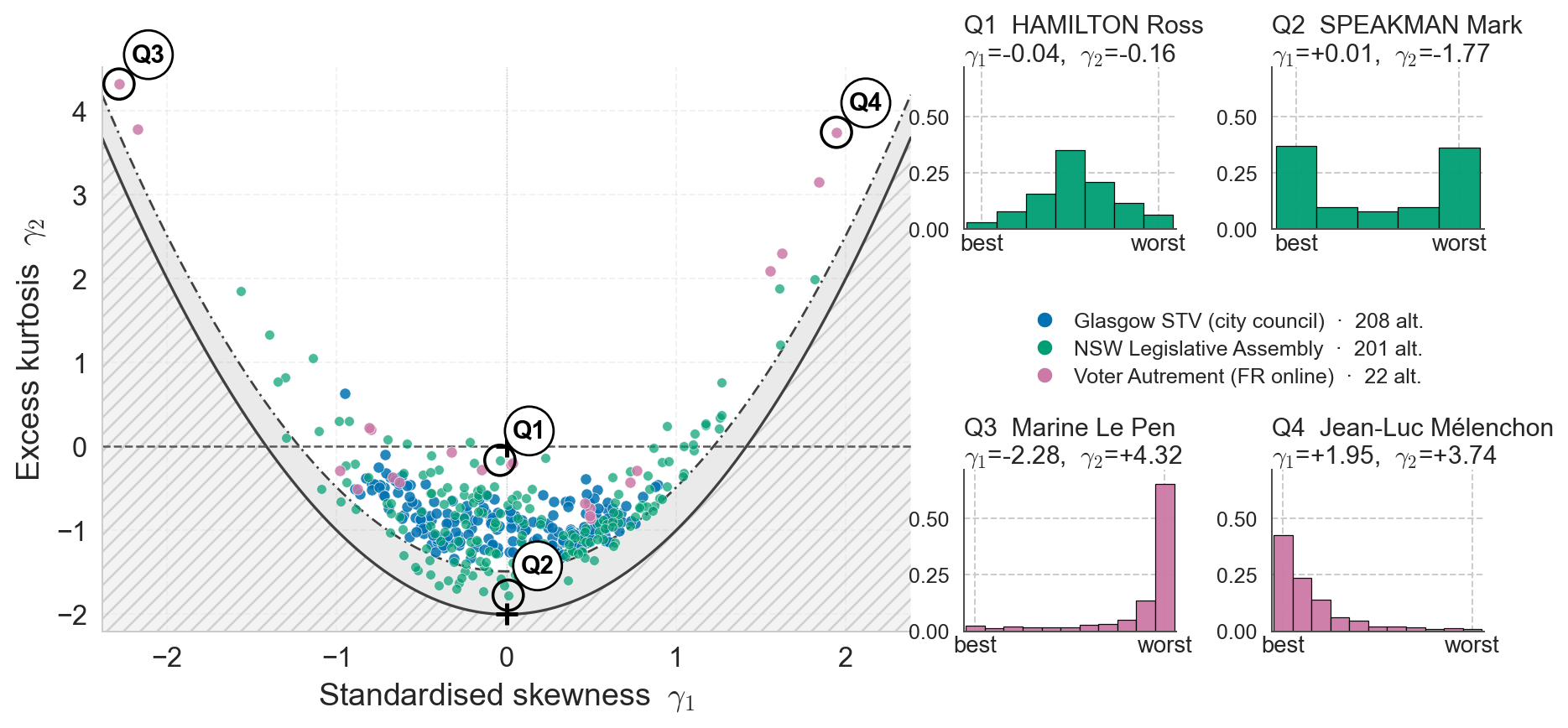}
  \caption{Pearson moment plane for three real-election datasets: Glasgow STV city council elections (208 alternatives), New South Wales Legislative Assembly (201 alternatives), and the 2017 Voter Autrement online survey (22 candidates). Solid and dash-dotted curves are the bounds from Figure~\ref{fig:synthetic}, with the upper envelope set for the largest~$m$. Four highlighted candidates illustrate the four regions of the plane: Q1 (Hamilton Ross, NSW), Q2 (Speakman Mark, NSW), Q3 (Marine Le Pen, Voter Autrement), Q4 (Jean-Luc M\'elenchon, Voter Autrement). Inset histograms show the rank distribution of each highlighted candidate.}
  \label{fig:real}
\end{figure}

\Cref{fig:real} shows that the moment plane is broadly populated. Real elections do not concentrate near IC: alternatives appear throughout the feasible region, including positions deep in the asymmetric upper corners. The four highlighted candidates illustrate the four regions identified in \Cref{fig:synthetic}.

Le Pen and M\'elenchon would both register as high-disagreement alternatives under any measure based on rank variance alone. The moment plane separates them. The sign of $\gamma_1$ tells which extreme of the rank scale concentrates the mass, and the magnitude of $\gamma_2$ tells how heavy the tail toward the opposite extreme is.

Speakman illustrates a different pattern. His rank distribution is bimodal, and his position close to the lower parabola flags this directly. Several other alternatives sit near the parabola in the NSW cloud. Two-camp polarization appears in real elections at non-trivial frequency, which justify the use of a bimodality test to characterize disagreements.

A caveat applies to all three datasets and most sharply to Voter Autrement: the figures use only voters who provided complete rankings. Ranking 22 candidates fully is cognitively demanding, and voters who do so are likely highly engaged and politically informed. Le Pen's position should be read with this filter in mind: her vote share in 2017 French presidential first rounds has been near $20\%$, considerably above what her moment-plane position would suggest, consistent with a selection effect in which voters providing complete rankings underrepresent her actual base.

\section{Inventory of measures and rules in the $S$-plurality hierarchy}
\label{app:master-inventory}

In the main text we showed that the rank variance $\Var(r_a)$, the Navarrete divisiveness, and the agreement index $A(\pi)$ admit closed forms in $S$-plurality coordinates, and that these forms place them at levels $3$, $3$, and $2$ of the hierarchy respectively. Beyond these three, a number of further disagreement measures, polarization indices, and aggregation rules from the social-choice literature can be analyzed in the same way: each is either expressible as a function of $\{\splur{S}{a}\}_{S \subseteq \calA,\, a \in S}$ at some finite level $\ell$, or provably outside the hierarchy. We collect the results of this analysis here.

Table~\ref{tab:master} catalogs every measure and rule we have placed. Subsec.~\ref{app:master-notation} fixes the notation needed to read it. Subsec.~\ref{app:master-formulas} gives the long-form expressions for the entries that the table abbreviates by equation number, with derivations where the closed form is new to this paper.

\begin{table}[h!]
\centering
\caption{Aggregation rules and disagreement measures expressed in $S$-plurality form, grouped by purpose and sorted by informational level. Long formulas are referenced by equation number; notations are in App.~\ref{app:master-notation}.}
\label{tab:master}
\scriptsize
\setlength{\tabcolsep}{3pt}
\renewcommand{\arraystretch}{1.4}
\setlength{\abovetopsep}{6pt}
\begin{tabular}{@{} >{\centering\arraybackslash}p{0.45cm} p{3.4cm} p{6.4cm} p{2.6cm} @{}}
\toprule
Lvl & Quantity & $S$-plurality form & Reference \\
\midrule
\multicolumn{4}{@{}l@{}}{\textbf{Aggregation measures}} \\
\addlinespace[1pt]
\multicolumn{4}{@{}l@{}}{\quad\textit{Level 2: pairwise, tournament-family}} \\
$2$ & Borda $\Bor(a)$ \emph{(score)} & $\sum_{b \neq a} \splur{\{a,b\}}{a}$ & \citep{borda1781} \\
$2$ & Copeland \emph{(rule)} & $\#\{b : p_{ab} > \tfrac12\} - \#\{b : p_{ab} < \tfrac12\}$ & \citep{copeland1951} \\
$2$ & Minimax \emph{(rule)} & $\min_{b \neq a} \splur{\{a,b\}}{a}$ & \citep{simpson1969,kramer1977} \\
$2$ & Kemeny \emph{(rule)} & $\arg\max_\lambda \sum_{a \succ_\lambda b} \splur{\{a,b\}}{a}$ & \citep{kemeny1959} \\
$2$ & Slater, Ranked Pairs, Schulze \emph{(rule)} & computable from the tournament $\{\splur{\{a,b\}}{a}\}_{a, b}$ & \citep{slater1961,tideman1987,schulze2011} \\
\addlinespace[2pt]
\multicolumn{4}{@{}l@{}}{\quad\textit{Level $k$ (parametric in $k$)}} \\
$k$ & $k$-wise Kemeny \emph{(rule)} & Eq.~\eqref{eq:kwise-kemeny} & \citep{gilbert2020setwise} \\
\addlinespace[2pt]
\multicolumn{4}{@{}l@{}}{\quad\textit{Level $m$: positional, runoff}} \\
$m$ & Plurality $\Pr[r_a = 1]$ \emph{(rule)} & $\splur{\calA}{a}$ & \citep{borda1781} \\
$m$ & Anti-plurality $1 - \Pr[r_a = m]$ \emph{(rule)} & Eq.~\eqref{eq:anti-plurality} & \citep{myerson1993} \\
%$m$ & Rank-$i$ frequency $\Pr[r_a = i]$ \emph{(per alt.)} & Eq.~\eqref{eq:rank-i} & This paper \\
$m$ & $k$-approval, $1 < k < m - 1$ \emph{(rule)} & $\sum_{i \leq k} \Pr[r_a = i]$, partial sum of Eq.~\eqref{eq:rank-i} & \\
$m$ & Bucklin \emph{(rule)} & smallest $k$ with $\sum_{i \leq k} \Pr[r_a = i] \geq \tfrac12$ & \citep{bucklin1911} \\
$m$ & STV / IRV \emph{(rule)} & recursive $\arg\min_{a \in S} \splur{S}{a}$ on shrinking $S$ & \citep{tideman1995single} \\
\midrule
\multicolumn{4}{@{}l@{}}{\textbf{Disagreement measures}} \\
\addlinespace[1pt]
\multicolumn{4}{@{}l@{}}{\quad\textit{Level 2: pairwise}} \\
$2$ & Agreement $A(\pi)$, polarization $1 - A(\pi)$ \emph{(profile)} & $\tfrac{1}{\binom{m}{2}} \sum_{\{a,b\}} |2\,\splur{\{a,b\}}{a} - 1|$ & \citep{alcaldeunzu2013,can2015measuring} \\
$2$ & Generalized polarization family \emph{(profile)} & $\sum_{\{a,b\}} f(\splur{\{a,b\}}{a})$ for a class of $f$ & \citep{can2017generalized} \\
$2$ & Diversity $\Delta_{\mathrm{KT}}(\pi)$ \emph{(profile)} & $2 \sum_{\{a,b\}} \splur{\{a,b\}}{a}\,(1 - \splur{\{a,b\}}{a})$ & \citep{hashemi2014measuring} \\
$2$ & Partitioning ratio $\alpha(a,b)$ \emph{(pair)} & $2 \min(p_{ab},\, 1 - p_{ab})$ & \citep{delemazure2024conflicting} \\
\addlinespace[2pt]
\multicolumn{4}{@{}l@{}}{\quad\textit{Level 3: triple, divisiveness and conflict}} \\
$3$ & Rank variance $\Var(r_a)$ \emph{(per alt.)} & Eq.~\eqref{eq:rank-variance} & \cite{kendall1939problem} \\
$3$ & Navarrete divisiveness \emph{(per alt.)} & Eq.~\eqref{eq:navarrete} & \citep{navarrete2024understanding} \\
$3$ & $\alpha$-divisiveness, $s = \Bor$ \emph{(per alt.)} & Eq.~\eqref{eq:alpha-div} & \citep{colley2023} \\
$3$ & Discrepancy $\beta(a,b)$ \emph{(pair)} & $\Delta(a,b) / (m-1)$ & \citep{delemazure2024conflicting} \\
$3$ & Group imbalance $\phi(a,b)$ \emph{(pair)} & $|\Bor(a) - \Bor(b)| \,/\, \Delta(a,b)$ & \citep{delemazure2024conflicting} \\
$3$ & Discrepancy balance $\gamma(a,b)$ \emph{(pair)} & Eq.~\eqref{eq:gamma} & \citep{delemazure2024conflicting} \\
$3$ & MaxSum (most-conflictual pair) \emph{(rule)} & Eq.~\eqref{eq:maxsum} & \citep{delemazure2024conflicting} \\
$3$ & MaxNash \emph{(rule)} & Eq.~\eqref{eq:maxnash} & \citep{delemazure2024conflicting} \\
$3$ & MaxSwap \emph{(rule)} & Eq.~\eqref{eq:maxswap} & \citep{delemazure2024conflicting} \\
$3$ & $p$-MaxPolar, $p > 0$ \emph{(rule)} & Eq.~\eqref{eq:maxpolar} & \citep{delemazure2024conflicting} \\
\addlinespace[2pt]
\multicolumn{4}{@{}l@{}}{\quad\textit{Outside the hierarchy}} \\
\textit{n/a} & $k$-Kemeny distance $\kappa_k$ ($k \geq 2$) and derived $P(E)$, $D(E)$ \emph{(profile)} & inexpressible in $\{\splur{S}{\cdot}\}$ & \citep{faliszewski2023} \\
\textit{n/a} & Spearman, Cayley distance-based diversity \emph{(profile)} & inexpressible in $\{\splur{S}{\cdot}\}$ & \citep{hashemi2014measuring} \\
\textit{n/a} & Support-based diversity (\# distinct rankings) \emph{(profile)} & inexpressible in $\{\splur{S}{\cdot}\}$ & \citep{hashemi2014measuring} \\
\textit{n/a} & Karpov diversity orderings \emph{(profile)} & inexpressible in $\{\splur{S}{\cdot}\}$ & \citep{karpov2017} \\
\bottomrule
\end{tabular}
\end{table}

\subsection{Notation}
\label{app:master-notation}

We work over $m$ alternatives $\calA$ and a profile $\pi$ on $\calA$.

\paragraph{$S$-plurality and pairwise.}
The $S$-plurality of $a$ on $S \ni a$ is
\begin{equation*}
\splur{S}{a} = \Pr_\pi\bigl[a \text{ ranked first in } S\bigr],
\qquad \sum_{a \in S} \splur{S}{a} = 1.
\end{equation*}
Pairwise proportions are the level-2 specialization,
\begin{equation*}
p_{ab} = \splur{\{a,b\}}{a} = \Pr_\pi[a \succ b],
\qquad p_{ab} + p_{ba} = 1.
\end{equation*}

\paragraph{Borda and aggregate $S$-plurality.}
\begin{align*}
\Bor(a) &= \sum_{b \neq a} p_{ab} = \E_\pi[\,m - r_a\,], \\
P_k(a) &= \sum_{T \subseteq \calA \setminus \{a\},\ |T| = k} \splur{T \cup \{a\}}{a},
\qquad P_0(a) = 1.
\end{align*}
$P_k(a)$ aggregates the $S$-plurality of $a$ over all $S$ of size $k+1$ containing $a$. Note $P_1(a) = \Bor(a)$ and $P_{m-1}(a) = \splur{\calA}{a}$.

\paragraph{Expected rank-distance.}
The expected absolute rank gap between $a$ and $b$, used throughout the level-3 disagreement entries, is
\begin{equation*}
\Delta(a,b) = \E_\pi\bigl[\,|r_a - r_b|\,\bigr].
\end{equation*}
Despite the shared symbol, $\Delta(a,b)$ is unrelated to the profile-level KT diversity $\Delta_{\mathrm{KT}}(\pi)$ in Table~\ref{tab:master}. Its closed form in $S$-plurality coordinates is given as Eq.~\eqref{eq:delta} below.

\paragraph{Conditional rank-distances.}
\begin{equation*}
\Delta^+(a,b) = \E_\pi\bigl[\,|r_a - r_b|\,\big|\, a \succ b\,\bigr],
\qquad
\Delta^-(a,b) = \E_\pi\bigl[\,|r_a - r_b|\,\big|\, b \succ a\,\bigr].
\end{equation*}
Closed forms are given as Eq.~\eqref{eq:delta-plus}--\eqref{eq:delta-minus}.

\paragraph{Covariance kernel.}
\begin{equation*}
K_b
= \sum_{c \neq a, b} \bigl(\splur{\{a,b,c\}}{a} - p_{ab}\, p_{ac}\bigr)
= \sum_{c \neq a, b} \Cov_\pi\bigl(\ind[a \succ b],\, \ind[a \succ c]\bigr).
\end{equation*}

\label{app:master-formulas}

The expressions below correspond to the equation references in Table~\ref{tab:master}.

\subsection{Aggregation rules.}
With $\operatorname{top}_\lambda(S)$ denoting the alternative ranked first in $S$ by ranking $\lambda$, the $k$-wise Kemeny rule of Gilbert et al. \cite{gilbert2020setwise} is
\begin{equation}
\arg\min_\lambda\, \sum_{\ell = 2}^{k} \sum_{|S| = \ell} \bigl(1 - \splur{S}{\operatorname{top}_\lambda(S)}\bigr).
\label{eq:kwise-kemeny}
\end{equation}
The level-$m$ inversion identity expresses the rank-$i$ frequency in terms of $\{P_k(a)\}_k$,
\begin{equation}
\Pr_\pi[r_a = i] = \sum_{k = m-i}^{m-1} (-1)^{k-m+i} \binom{k}{m-i}\, P_k(a),
\label{eq:rank-i}
\end{equation}
from which the anti-plurality complement follows as the $i = m$ specialization,
\begin{equation}
1 - \Pr_\pi[r_a = m] = 1 - \sum_{k=0}^{m-1} (-1)^k\, P_k(a).
\label{eq:anti-plurality}
\end{equation}

\subsection{Rank variance.}
\begin{equation}
\Var_\pi(r_a) = \Bor(a)\bigl(1 - \Bor(a)\bigr) + 2 \sum_{\{b,c\} \subseteq \calA \setminus \{a\}} \splur{\{a,b,c\}}{a}.
\label{eq:rank-variance}
\end{equation}

\subsection{Navarrete et al. Divisiveness.}
\begin{equation}
\Div_{\mathrm{Nav}}(a)
= \frac{1}{m-1} \sum_{b \neq a} \biggl| 1 + \frac{K_b}{p_{ab}\,(1 - p_{ab})} \biggr|.
\label{eq:navarrete}
\end{equation}

\subsection{$\alpha$-Divisiveness.}
The $\alpha$-divisiveness \cite{colley2023} with Borda scoring is
\begin{equation}
\Div_\alpha^{\Bor}(a)
= \frac{1}{m-1} \sum_{b \neq a} \bigl(p_{ab}\,(1 - p_{ab})\bigr)^\alpha\, \bigl|\Bor(a \mid a \succ b) - \Bor(a \mid b \succ a)\bigr|,
\label{eq:alpha-div}
\end{equation}
where the conditional Borda scores expand as
\begin{align*}
\Bor(a \mid a \succ b) &= 1 + \frac{1}{p_{ab}} \sum_{c \neq a, b} \splur{\{a,b,c\}}{a}, \\
\Bor(a \mid b \succ a) &= \frac{1}{1 - p_{ab}} \sum_{c \neq a, b} \bigl(p_{ac} - \splur{\{a,b,c\}}{a}\bigr).
\end{align*}

\subsection{Expected rank-distance.}
\begin{align}
\Delta(a,b)
&= 2(m-1) - \Bor(a) - \Bor(b) - 2 \sum_{c \neq a, b} \splur{\{a,b,c\}}{c},
\label{eq:delta}
\end{align}
obtained from the triple-ordering identity
\begin{equation*}
\Pr_\pi\bigl[c \text{ between } a \text{ and } b\bigr]
= 2 - 2\,\splur{\{a,b,c\}}{c} - p_{ac} - p_{bc},
\end{equation*}
summed over $c \neq a, b$, together with $\sum_{c \neq a,b} p_{ac} = \Bor(a) - p_{ab}$.

\subsection{Conditional rank-distances.}
\begin{align}
\Delta^+(a,b) &= \frac{\Delta(a,b) + \Bor(a) - \Bor(b)}{2\, p_{ab}}, \label{eq:delta-plus}\\
\Delta^-(a,b) &= \frac{\Delta(a,b) - \Bor(a) + \Bor(b)}{2\,(1 - p_{ab})}. \label{eq:delta-minus}
\end{align}
The first equality follows from total expectation, $\Delta(a,b) = p_{ab}\,\Delta^+ + (1-p_{ab})\,\Delta^-$; the closed form combines this with the signed identity $\Bor(a) - \Bor(b) = p_{ab}\,\Delta^+ - (1-p_{ab})\,\Delta^-$.

\subsection{Discrepancy balance.}
\begin{equation}
\gamma(a,b) = \min\!\left(\frac{\Delta^+(a,b)}{\Delta^-(a,b)},\; \frac{\Delta^-(a,b)}{\Delta^+(a,b)}\right).
\label{eq:gamma}
\end{equation}

\subsection{Conflictual rule scores.}
Substituting~\eqref{eq:delta}--\eqref{eq:delta-minus} into the definitions of Delemazure et al. \cite{delemazure2024conflicting} reduces each conflictual rule to a function of $\Delta(a,b)$, $p_{ab}$, $\Bor(a)$, and $\Bor(b)$. Each rule selects $\arg\max_{\{a,b\}}$ of its score:
\begin{align}
\text{MaxSum}(a,b) &= \Delta(a,b) + (1 - 2\,p_{ab})\bigl(\Bor(a) - \Bor(b)\bigr), \label{eq:maxsum}\\
\text{MaxNash}(a,b) &= \Delta(a,b)^2 - \bigl(\Bor(a) - \Bor(b)\bigr)^2, \label{eq:maxnash}\\
\text{MaxSwap}(a,b) &= \Delta(a,b) - \bigl|\Bor(a) - \Bor(b)\bigr|, \label{eq:maxswap}\\
p\text{-MaxPolar}(a,b) &= \min(p_{ab},\, 1 - p_{ab}) \cdot \Delta(a,b)^p. \label{eq:maxpolar}
\end{align}
All four scores depend on the level-3 data $\{\splur{\{a,b,c\}}{c}\}_{c \neq a,b}$ through $\Delta(a,b)$ in~\eqref{eq:delta}, together with the level-2 data $\{p_{ab}, \Bor(\cdot)\}$.

\end{document}